\documentclass[journal]{IEEEtran}
\usepackage{amsmath, amssymb, amsthm}
\usepackage[normalem]{ulem}
\usepackage{listings}
\usepackage{indentfirst}
\usepackage{graphicx}
\usepackage{float}
\usepackage{extarrows}
\usepackage{epstopdf}
\usepackage{bbm}
\usepackage{caption}
\usepackage{subcaption}
\usepackage{todonotes}
\usetikzlibrary{automata,positioning}
\usetikzlibrary{arrows,shapes,chains, intersections}
\usepgflibrary{patterns}
\usepackage{algorithm}
\usepackage{algorithmic}
\usepackage{tikz,pgfplots}
\usepackage{paralist}
\usepackage[sort&compress,numbers]{natbib}
\usepackage{wrapfig}
\usepackage{comment}
\pgfplotsset{plot coordinates/math parser=false}
\usepackage{enumitem}
\usepackage[utf8]{inputenc} 
\usepackage[T1]{fontenc}    
\usepackage{hyperref}       
\usepackage{url}            
\usepackage{booktabs}       
\usepackage{amsfonts}       
\usepackage{nicefrac}       
\usepackage{microtype}      
\usepackage{xcolor} 
\usepackage{xspace}
\usetikzlibrary{decorations.pathreplacing}
\setlist[itemize]{leftmargin=*}

\IEEEoverridecommandlockouts  

\newtheorem{assumption}{Assumption}
\newtheorem{observation}{Observation}

\theoremstyle{definition}
\newtheorem{definition}{Definition}

\newtheorem{lemma}{Lemma}
\newtheorem{theorem}{Theorem}

\theoremstyle{remark}
\newtheorem{remark}{Remark}

\newcommand{\set}[1]{\left\lbrace#1\right\rbrace}
\newcommand{\setIn}[1]{\mathbbm{1}_{\set{#1}}}

\newcommand{\E}{\mathbb{E}}
\newcommand{\N}{\mathbb{N}}

\newcommand{\cK}{\mathcal{K}}

\newcommand{\cS}{\mathcal{S}}
\newcommand{\cT}{\mathcal{T}}
\newcommand{\cV}{\mathcal{V}}
\newcommand{\cW}{\mathcal{W}}
\newcommand{\cX}{\mathcal{X}}
\newcommand{\cZ}{\mathcal{Z}}

\newcommand{\Prb}{\mathrm{Pr}}
\renewcommand{\ge}{\geqslant}
\renewcommand{\le}{\leqslant}

\newcommand{\MAB}{\textsc{Aging Bandit with Adaptive Reset}\xspace}
\newcommand{\mab}{\textsc{ABAR}\xspace}

\title{Adaptive Requesting in Decentralized Edge Networks via Non-Stationary Bandits}
\author{Yi Zhuang\IEEEauthorrefmark{1}, 
Kun Yang\IEEEauthorrefmark{2}, \IEEEmembership{Fellow, IEEE},
Xingran Chen\IEEEauthorrefmark{3}, \IEEEmembership{Member, IEEE}
\IEEEcompsocitemizethanks 
{
\IEEEcompsocthanksitem Yi Zhuang is with School of Information and Communication Engineering, University of Electronic Science and Technology of China, Chengdu, Sichuan, China (Email: yizhuang265@163.com).
\IEEEcompsocthanksitem Kun Yang is with School of Computer Science and Electronic Engineering, University of Essex, Colchester, UK (Email: kunyang@essex.ac.uk).
\IEEEcompsocthanksitem Xingran Chen is with Department of Electrical and Computer Engineering, Rutgers University, Piscataway Township, NJ, USA (Email: xingranc@ieee.org). \textit{(Corresponding Author: Xingran Chen)}
}
} 

\begin{document}
\maketitle

\begin{abstract} 
We study a decentralized collaborative requesting problem that aims to optimize the information freshness of time-sensitive clients in edge networks consisting of multiple clients, access nodes (ANs), and servers. Clients request content through ANs acting as gateways, without observing AN states or the actions of other clients.
We define the reward as the age of information reduction resulting from a client’s selection of an AN, and formulate the problem as a non-stationary multi-armed bandit.
In this decentralized and partially observable setting, the resulting reward process is history-dependent and coupled across clients, and exhibits both abrupt and gradual changes in expected rewards, rendering classical bandits ineffective. To address these challenges, we propose the \MAB algorithm, which combines adaptive windowing with periodic monitoring to track evolving reward distributions. We establish theoretical performance guarantees showing that the proposed algorithm achieves near-optimal performance, and we validate the theoretical results through simulations.
\end{abstract}


\section{Introduction}\label{sec:Intro}

The proliferation of latency-sensitive applications, such as real-time sensing~\cite{Kaul2011AoI}, interactive services~\cite{8262777}, and distributed control~\cite{AoISurvey}, poses a significant challenge to maintaining the freshness of information in modern computer networks, as the utility of such applications critically depends on the timely delivery of updates. Traditional cloud-centric networks rely on centralized processing, in which data generated at the network edge must be transmitted to remote cloud servers for computation and decision making~\cite{6553297}. This centralized workflow introduces long communication paths and concentrates traffic on backhaul links, which, in turn, leads to increased transmission delays and network congestion. As a result, cloud-centric networks often struggle to meet the stringent latency requirements imposed by latency-sensitive applications, significantly degrading information freshness \cite{cxrFRAN, AoIrandomaccess}.

To overcome these limitations, modern network designs are increasingly shifting toward decentralized edge networks \cite{8016573} that distribute  computation, storage, and control across the network \cite{cxrFRAN}. Such networks typically consist of end users, access nodes (ANs), and servers, where ANs are deployed closer to end users and serve as intermediate network entities. Within this framework, ANs function as gateways that perform localized caching and forwarding of information, enabling data to be processed and delivered without always traversing the entire network to the cloud. By shortening communication paths and alleviating backhaul congestion, this decentralized network effectively reduces end-to-end latency and allows time-critical updates to be delivered to clients in a more timely manner \cite{cxrFRAN}.

This paper addresses the issue of timely content requests for latency-sensitive end users, referred to as clients, in a decentralized edge network (see Fig.~\ref{fig:Example}). In this setting, multiple servers, ANs, and clients interact within the edge network, where clients cannot directly communicate with servers. Instead, ANs act as gateways, either fetching cached content or sending commands to servers for content retrieval \cite{eRRHgateways, cxrFRAN}. To ensure timely information delivery, we adopt the Age of Information (AoI) \cite{Kaul2011AoI, AoISurvey} as the performance metric and aim to minimize the time-average AoI of clients, where AoI quantifies the time elapsed since the generation of the most recently received update. Optimizing AoI in a decentralized edge network therefore has broad implications for real-time applications---including smart cities, autonomous vehicles, industrial automation, and health monitoring systems \cite{AoIrandomaccess, AoIEstimation, AoICache}---where low-latency and fresh information are essential for safety and operational efficiency.

\begin{figure}[htbp]
\centering
\includegraphics[width=0.6\linewidth, height=0.5\linewidth]{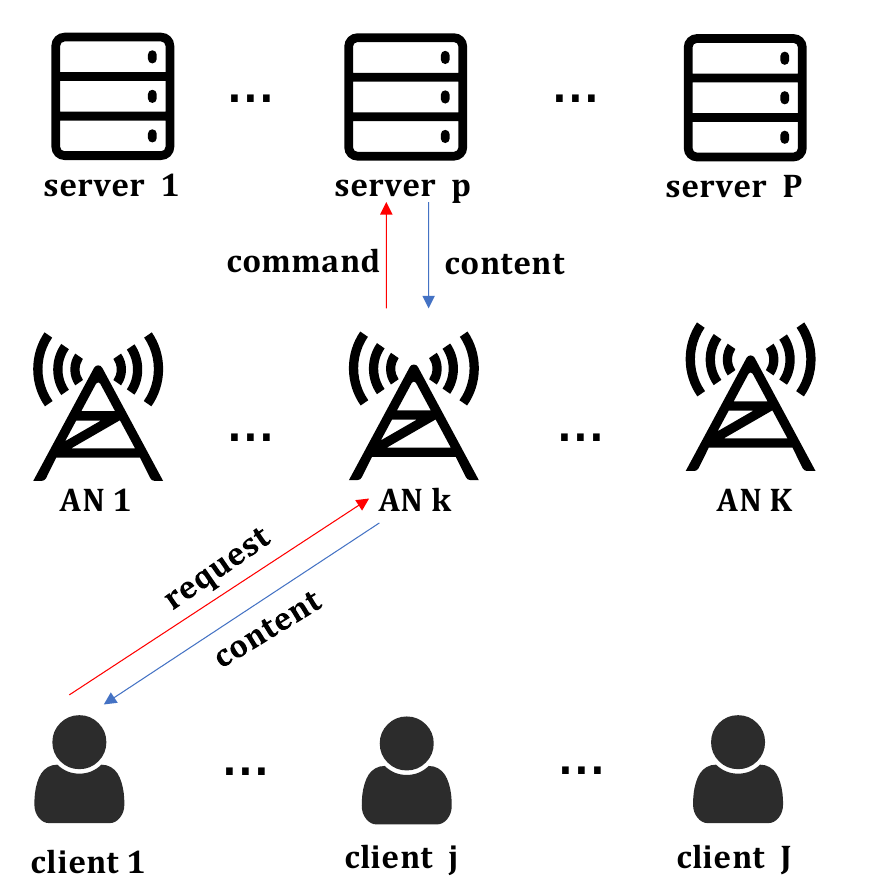}
\caption{ An example of decentralized an edge network.}
\label{fig:Example}
\end{figure}

This problem presents several key challenges. Some related challenges have been partially discussed in prior work~\cite{cxrFRAN}, but they are systematically addressed here. (i) \textit{Decentralization among clients}: each client makes decisions based solely on local information, hindering global coordination across the network. (ii)  \textit{ Intra-client decision coupling}: since each client can send at most one request per server per time slot, its decisions across different servers are inherently coupled. This coupling significantly complicates the analysis and renders existing tools---such as the age-of-version metric introduced in prior work~\cite{AgeVersion}---inapplicable. (iii) \textit{Inter-client coupling via shared ANs}: a client's request affects the state of shared ANs, which in turn alters the observations available to other clients. This creates a dynamically coupled environment, making real-time decision-making analytically intractable. (iv) \textit{Topological complexity}: the decentralized edge network exhibits a two-hop network with multiple receivers per hop and limited communication resources. This structural complexity further complicates policy optimization, as coordination must be achieved under stringent resource constraints.

Given the challenges above, traditional theoretical analysis becomes infeasible, motivating the use of reinforcement learning techniques, which have demonstrated strong performance in real-time decision-making tasks \cite{RLbook}. In decentralized edge networks, it is essential to develop decentralized learning that allows individual nodes---each with access only to local observations---to collaboratively optimize global objectives without centralized coordination \cite{DecLearn}. Among such approaches, Multi-Armed Bandit (MAB) methods \cite{MABbook} are particularly promising due to their simplicity, scalability, and favorable analytical properties, making them well-suited for distributed online decision-making under uncertainty.

According to the definition of AoI~\cite{AoISurvey}, the age increases when no packet is delivered and drops upon the reception of a fresh packet. As a result, the state of each arm---typically including the age---evolves over time even when it is not selected. This ``restless'' evolution induces highly non-stationary, history-dependent reward dynamics and naturally characterizes the problem as a restless bandit~\cite{restlessBandit} (referred to as an \textit{Aging Bandit} problem~\cite{Agingbandit}).

Although non-stationary MAB variants (e.g., SW-UCB and D-UCB) have been developed~\cite{garivier2008upper}, they remain insufficient for our setting for three reasons: (i) they often require prior knowledge of the change frequency; (ii) they typically assume independent rewards across agents, which is violated by our shared-update mechanism; and (iii) they are usually tailored to either abrupt or gradual changes, but not both simultaneously.

A well-known approach for restless bandits is Whittle's index policy, which has been widely applied in Aging Bandit settings~\cite{8437712,9241401,10666836,11015494,9579440,8493069,10462087}. However, Whittle-type approaches are ill-suited for decentralized decision-making because they generally require global state information \cite{8493069}. Moreover, verifying Whittle indexability is often challenging in complex decentralized environments; prior work~\cite{8493069} only extends the framework to a limited and simplified decentralized case that does not cover our setting.

Decentralized Aging Bandit formulations have also been studied in~\cite{Agingbandit,9559999,9668362}, primarily for dynamic channel selection in single-hop wireless networks to improve information freshness. These works focus on estimating channel erasure probabilities with rewards constrained to $[0,1]$, and their modeling assumptions and objectives differ substantially from ours, limiting their applicability to the decentralized edge-network setting considered here.

\subsection{Contributions}\label{subsec:Contribution}

In this work, we study the problem of optimizing information freshness for time-critical clients in decentralized edge networks. Each client independently selects an AN based solely on its local observations, without access to the states of ANs or the actions of other clients. By defining the reward as the instantaneous reduction in AoI, we formulate the coordinated request optimization problem as a \textit{decentralized Aging Bandit} problem with highly non-stationary and correlated rewards, featuring both abrupt and gradual environmental changes.

To address the challenges arising from this setting, we make the following key contributions:

\begin{compactenum}[(i)]
\item \textbf{Age-based Reward Design and Bandit Reformulation}.  
We define the reward of each action as the instantaneous reduction in age, such that an action receives a higher reward when it makes the information fresher. Under this reward definition, minimizing the time-average AoI is reduced to maximizing the cumulative reward over time. This reformulation enables us to cast the original information freshness optimization problem as a non-stationary multi-armed bandit problem.

\item \textbf{Decentralized Algorithm for Non-Stationary Aging Bandits.}  
Unlike existing Aging Bandit algorithms~\cite{Agingbandit,9559999,9668362}, which focus on settings with rewards constrained to $[0,1]$ and independent across agents, we design a decentralized algorithm, termed \MAB (\mab), for Aging Bandits with non-stationary and correlated rewards. The proposed algorithm combines adaptive windowing with a monitoring-based reset strategy so that each client can locally detect when the reward dynamics change and react accordingly. As a result, the algorithm can cope with reward variations caused by its own decisions, other clients’ actions, and ANs' update behaviors, without relying on global coordination or prior knowledge of environmental dynamics.

\item \textbf{Theoretical Guarantees.}  
We develop a theoretical framework for non-stationary multi-armed bandits that accommodates environments in which abrupt and gradual changes coexist. The existing framework (e.g., ADR~\cite{JMLR}) typically relies on simplified assumptions that the environment exhibits either purely abrupt or purely gradual changes. In this work, we systematically extend this framework to a more general non-stationary setting with mixed change dynamics. Based on the proposed theoretical framework, we prove that the proposed algorithm is asymptotically optimal, achieving sub-linear regret over time. Extensive simulations further corroborate the theoretical findings.

\end{compactenum}

\subsection{Notation}\label{subsec:Notations}
We use the notation $\E[\cdot]$ and $\Prb(\cdot)$ to denote expectation and probability, respectively.  The index sets $[J] = \set{1, 2, \dots, J}$, $[K] = \set{1, 2, \dots, K}$, and $[P] = \set{1, 2, \dots, P}$ represent the sets of clients, ANs, and servers, respectively. Let $T$ denote the time horizon. The indicator function $\setIn{A}$ equals $1$ if the event $A$ occurs, and $0$ otherwise.  The functions $h_{jp}(t)$ and $g_{kp}(t)$ denote the age of information of server~$p$ at client~$j$ and at AN~$k$ at time slot~$t$, respectively. The notation $\mathcal{O}(\cdot)$ follows the Bachmann--Landau convention and represents Big-O asymptotic bounds.

The rest of the paper is organized as follows. Section~\ref{sec:ProFor} introduces the system model and problem formulation. Section~\ref{sec:OptMAB} defines the AoI-based reward and reformulates the original problem as a non-stationary multi-armed bandit framework. Section~\ref{sec:ProAlg} presents the proposed \mab algorithm. Theoretical guarantees are established in Section~\ref{sec:assumptions} and Section~\ref{sec:ProAnl}. Simulation results are reported in Section~\ref{sec:NumRes}. Finally, Section~\ref{sec:Conclusion} concludes the paper.


\section{Problem Formulation}\label{sec:ProFor}
\subsection{Network Model}
We consider a decentralized network consisting of $J$ clients, $K$ access nodes (ANs), and $P$ servers. Clients correspond to end devices (e.g., smartphones, laptops, or IoT terminals) that issue time-sensitive content requests and require up-to-date packet updates. ANs act as edge nodes equipped with local caching and forwarding capabilities, while servers are content sources responsible for generating the latest content to meet user demands. 
The sets of clients, ANs, and servers are denoted by $[J] = \set{1, 2, \dots, J}$, $[K] = \set{1, 2, \dots, K}$, and $[P] = \set{1, 2, \dots, P}$, respectively. An example of the network is illustrated in Fig.~\ref{fig:SysMod}. In this system, clients cannot communicate with servers directly; instead, ANs serve as gateways between clients and servers. 

Since content is transmitted in the form of packets, we use the terms \textit{content} and \textit{packets} interchangeably. Each AN is capable of caching and forwarding packets. When client~$j$ requests the most recent packet from server~$p$, the request is forwarded to an AN, denoted by AN~$k$. Upon receiving the request, AN~$k$ may either serve the packet from its local cache or command server~$p$ to generate and transmit a fresh packet.

\begin{figure}[htbp]
\centering
\includegraphics[width=0.6\linewidth, height=0.5\linewidth]{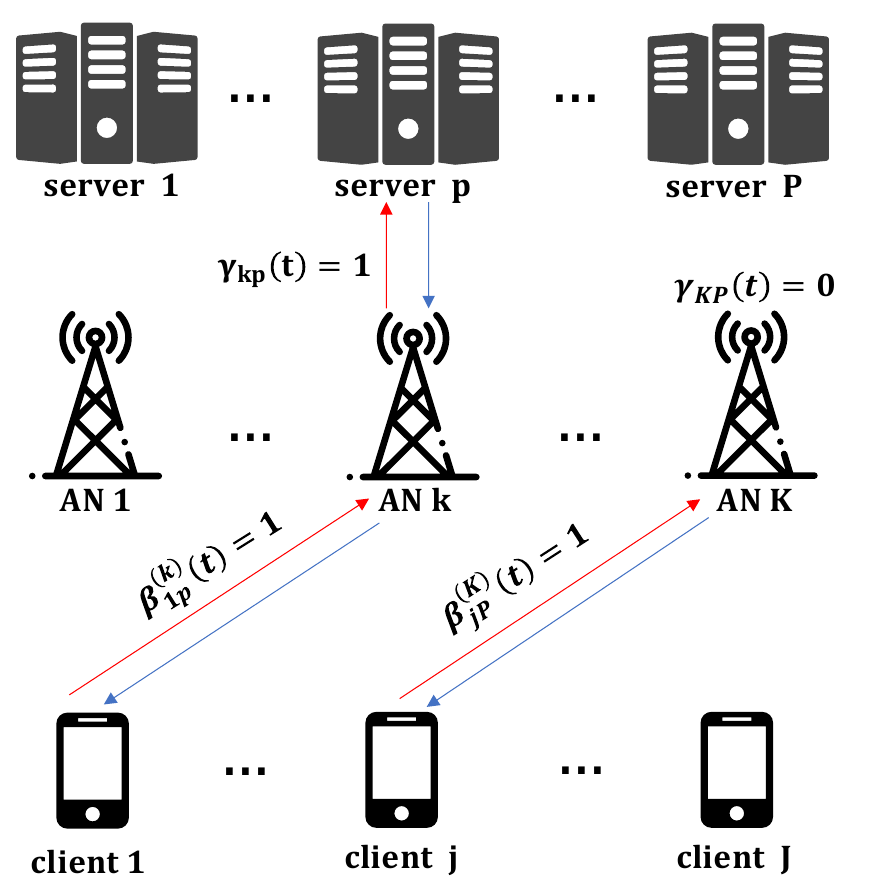}
\caption{Two illustrative service scenarios in a decentralized network. In the first, client~$1$ requests content from server $p$ via AN $k$. Upon receiving the request, AN~$k$ decides to command the server~$p$ to generate a new packet. In the second, another client~$j$ requests content from server~$P$ via AN~$K$, and the AN serves the request directly from its local cache.}
\label{fig:SysMod}
\end{figure}

We consider a slotted-time system indexed by $t\in [T]$. Let $\beta_{jp}^{(k)}(t) \in \set{0, 1}$ denote whether client~$j$ sends a request for content from server $p$ via an AN~$k$ in time slot $t$. Specifically, $\beta_{jp}^{(k)}(t) = 1$ indicates that such a request is sent, and $\beta_{jp}^{(k)}(t) = 0$ otherwise. We assume that, for any  client-server pair $(j, p)$, the client sends its request through at most one AN in each time slot. This is captured by the following constraint:
\begin{align}\label{eq:request at most one}
\sum_{k \in [K]} \beta_{jp}^{(k)}(t) \le 1,\,\, \forall\, j \in [J],\,p \in [P],\, t \in [T].
\end{align}
It is worth noting that $\set{\beta_{jp}^{(k)}(t)}_{k}$ are not independent over $k$. This inter-dependence stems from the fact that at most one request can be sent among the $K$ available ANs in each time slot. This structural dependence introduces additional complexity compared to previous studies \cite{AoICache, 9771060, AgeVersion, 9917351, 10349704}, where request decisions are typically modeled as independent across nodes.

Let $b_{jp}^{(k)}(t)$ denote the probability that client $j$ sends a request to server $p$ via  AN~$k$ at time $t$, i.e., 
\begin{align}\label{eq:request prob}
b_{jp}^{(k)}(t) \triangleq \Prb\left(\beta_{jp}^{(k)}(t) = 1\right).
\end{align}
From~\eqref{eq:request at most one} and~\eqref{eq:request prob}, we obtain
\begin{align}\label{eq:RqstOneProb}
\sum_{k \in [K]} b_{jp}^{(k)}(t) \le 1,\,\, \forall\, j \in [J],\,p \in [P],\, t \in [T].
\end{align}

Upon receiving content requests from clients, AN~$k$ determines how to serve requests associated with server $p$. Specifically, when a request for server $p$ is present at time slot $t$, AN $k$ decides whether to fetch a fresh packet from server $p$ or to serve the request using a locally cached copy. Let $\gamma_{kp}(t) \in \set{0, 1}$ denote the decision of AN~$k$ for server~$p$ at time slot $t$. Here, $\gamma_{kp}(t) = 1$ indicates that AN~$k$ commands server~$p$ to generate and transmit a fresh packet, whereas $\gamma_{kp}(t) = 0$ indicates that AN~$k$ uses the cached packet. 

We assume that, conditioned on the presence of requests, the decisions $\set{\gamma_{kp}(t)}_p$ are independent across $p$. This assumption is reasonable, as requests for different servers are typically independent in practice, which reflects real-world implementations where ANs update content streams independently. Such an assumption simplifies the system model while retaining practical relevance. We define:
\begin{align}\label{eq:gamma_prob}
r_{kp} \triangleq \Prb\left( \gamma_{kp}(t) = 1 \mid \beta_{jp}^{(k)}(t)=1\right).
\end{align}
which specifies the \textit{fixed} probability that AN~$k$ requests a fresh packet from server~$p$ at any time slot.\footnote{This formulation can be readily generalized to time-varying update probabilities $r_{kp}(t)$.} Considering that AN~$k$ obtains an update from server~$p$ if at least one client~$j$ requests content associated with server~$p$ at time $t$, the update probability is given by:
\begin{align}\label{eqCPupdateprob}
\left(1 - \prod_{j \in [J]} \left( 1 - b_{jp}^{(k)}(t) \right) \right) \cdot r_{kp}.
\end{align}
The term $\prod_{j \in [J]} \left( 1 - b_{jp}^{(k)}(t) \right)$ corresponds to the probability that no client makes a request for server~$p$ at AN~$k$.

To model the limited resources of each AN, we impose a resource constraint on its update decisions. Let $R$ denote the maximum resources that each AN can utilize per time slot:
\begin{align}\label{eq:eRRHConstraint}
\sum_{p \in [P]} r_{kp} \le R,\,\, \forall\, k \in [K].
\end{align}

For clarity, we explicitly state the following key assumptions in our system model: 
\begin{compactenum}[(i)]
\item Client requests and AN update commands are small in size, and their transmission delays are therefore negligible.
\item Fetching a fresh packet from a server or serving a request using locally cached content at an AN each requires exactly one time slot. The transmission delay from ANs to clients is also assumed to be negligible.
\item Interference is ignored, consistent with prior studies~\cite{AoICache, 9771060, AgeVersion, 9917351, 10349704}, as it can be effectively mitigated using PD-NOMA or similar multiple-access techniques.
\item When multiple clients request the same content from server~$p$ via AN~$k$ in the same time slot, AN $k$ adopts a single update decision---either fetching a fresh packet or serving all requests using the cached copy.
\item Each AN can simultaneously serve multiple client requests without incurring queuing delays, enabled by parallel processing.
\end{compactenum}

\subsection{Age of Information}
To capture the timeliness of content, we adopt the age-of-information (AoI) metric~\cite{Kaul2011AoI, AoISurvey} at both the ANs and the clients. Following the modeling paradigm in~\cite{cxrFRAN, AoIrandomaccess}, we consider two types of AoI: AoI defined at the ANs and AoI defined at the clients. Both AoI processes are updated at the end of each time slot.

At time slot $t$, let $\tau_{kp}$ denote the generation time of the most recently received packet from server~$p$ at AN $k$ \textit{prior to} time $t$. The AoI of server~$p$ at AN~$k$ is then defined as
\begin{align*}
g_{kp}(t) = t - \tau_{kp}.
\end{align*}
According to the model assumptions, fetching a new packet from a server requires exactly one time slot. Consequently, the AoI process $\set{g_{kp}(t)}$ evolves as
\begin{align}\label{eq:RecurAgeeRRH}
g_{kp}(t + 1) = & \setIn{\gamma_{kp}(t) \sum_{j} \beta_{jp}^{(k)}(t)>0}\nonumber\\
+& \left(g_{kp}(t) + 1\right)\setIn{\gamma_{kp}(t)\sum_{j} \beta_{jp}^{(k)}(t)=0}.
\end{align}
That is, the AoI is reset to $1$ when AN~$k$ successfully fetches a fresh packet from server~$p$ at time $t$, and increases by $1$ otherwise. The initial condition is given by $g_{kp}(0) = 1$.

Similarly, let $\tau'_{jp}$ denote the generation time of the most recently received packet from server~$p$ at client~$j$ \textit{prior to} time $t$. The corresponding AoI is defined as
\begin{align*}
h_{jp}(t) = t - \tau'_{jp}.
\end{align*}
Under our model assumptions, the transmission delay from ANs to clients is negligible. Furthermore, to ensure data freshness, any packet that is older than the most recently received one is discarded upon delivery. Accordingly, the AoI process $\set{h_{jp}(t)}$ evolves as follows:
\begin{align}\label{eq:REcurAgeClient}
h_{jp}(t+1) =&\sum_{k\in[K]}\Bigg(\setIn{\beta_{jp}^{(k)}(t) \gamma_{kp}(t) = 1}\nonumber\\
+&\left(\tilde{h}_{jp}^{(k)}(t) + 1\right)\setIn{\beta_{jp}^{(k)}(t)\Big(1-\gamma_{kp}(t)\Big) = 1} \Bigg)\nonumber\\
+&\left({h}_{jp}(t) + 1\right)\setIn{\sum_{k\in[K]}\beta_{jp}^{(k)}(t)= 0}.
\end{align}
where $\tilde{h}_{jp}^{(k)}(t) = \min\{h_{jp}(t), g_{kp}(t)\}$ and $h_{jp}(0) = 1$.

The long-term time-average age of information across all clients is defined as follows:
\begin{align}\label{eq:Avg_Age_CU}
J(T) = \frac{1}{T} \sum_{t = 1}^{T} \frac{1}{JP} \sum_{j \in [J]} \sum_{p \in [P]} \mathbb{E}[h_{jp}(t)].
\end{align}

\subsection{Objectives and Policies}\label{subsec:ObjPolicies}
We consider decentralized policies, under which each client makes decisions based solely on its local information. 
Under this decentralized setting, $\set{\beta^{(k)}_{jp}(t)}_{k,p, t}$ are independent across $j$. A decentralized policy is defined as
\begin{align}\label{eq:Policy}
\pi=\left\{ \left\{b_{jp}^{(k)}(t) \right\}_{k,p,t}\right\}_{j}.
\end{align}
Our objective is to minimize the time-average AoI of clients over a time horizon $T$. The optimization problem is formulated as:
\begin{align}\label{eq:Opt0}
\text{(P1)}\;\; & \min_{\pi} \;\; J(T)\nonumber\\
&\;\;\text{s.t.}\;\; \eqref{eq:RqstOneProb}
\end{align}
where $J(T)$ is defined in \eqref{eq:Avg_Age_CU} and $\pi$ is defined in \eqref{eq:Policy}.

\section{From Optimization to Multi-Armed Bandits}\label{sec:OptMAB}
\subsection{Challenges}\label{subsec:Challenges}
The optimization problem \eqref{eq:Opt0} involves several key challenges that fundamentally limit the application of classical decision-making approaches---such as dynamic programming \cite{powell2007approximate}, Lyapunov optimization \cite{neely2010stochastic}, classical MDP policies \cite{puterman2014markov}.

The first challenge lies in \textit{real-time} decision making. In our system, client-side policies may vary over time in response to rapidly changing network environment (e.g., user demand). As a result, the induced AoI processes are generally non-stationary and may not be ergodic. However, classical optimization methods---such as dynamic programming, Lyapunov optimization, and MDP-based policies---typically rely on stationary system dynamics or stable long-term statistical properties to guarantee performance optimality. These assumptions are violated in our setting, rendering classical approaches inapplicable and motivating the need for new optimization techniques.

The second challenge arises from \textit{decentralized} decision making. As discussed in~\cite{cxrFRAN}, there is no central scheduler coordinating the actions of clients. Instead, each client independently makes decisions based solely on local observations. This lack of global information and coordination significantly complicates the analysis and renders many traditional theoretical methods inapplicable.

The third challenge is \textit{partial observability}. Limited client-side observations prevent the use of full-information learning methods, including Q-learning \cite{watkins1992q}, that rely on complete state or transition information.

Due to the challenges discussed above, classical decision-making approaches are not applicable.

\subsection{Myopic Reformulation and Reward Design}\label{subsec:RewardDesign}

To enable tractable decentralized decision-making, we adopt a \textit{myopic} optimization approach that optimizes original problem \eqref{eq:Opt0}, following the idea in~\cite{10666836, chen2023containing}.

Specifically, instead of minimizing the long-term cumulative AoI, we seek to solve the following optimization sequentially in time for $0\le t\le T-1$:
\begin{align}\label{eq:Opt1}
\text{(P2)}\;\; & \min_{\pi} \;\; \frac{1}{TJP}\sum_{\substack{j\in[J]\\p\in[P]}}\E\left[h_{jp}(t)\right]\nonumber\\
&\;\;\text{s.t.}\;\; \eqref{eq:RqstOneProb}
\end{align}
which aims to minimize the expected increase in AoI over the current time slot.

Following a formulation similar to~\cite{Agingbandit}, we define a slot-based reward that directly captures the \textit{instantaneous reduction} in AoI caused by each action. According to the age recursion in~\eqref{eq:REcurAgeClient}, the age of a client increases linearly by one in the absence of a successful update, and resets to a smaller value, either $1$ or $\tilde{h}_{jp}^{(k)}(t)$, upon a successful content delivery. Motivated by this observation, the reward associated with client~$j$ and server~$p$ through AN~$k$, denoted by $x_{jp}^{(k)}(t)$, is defined as
\begin{align}\label{eq:reward_design}
x_{jp}^{(k)}(t) =&h_{jp}(t)\setIn{\beta_{jp}^{(k)}(t) \gamma_{kp}(t) = 1} \nonumber\\
+&\left(h_{jp}(t)-\tilde{h}_{jp}^{(k)}(t)\right)\setIn{\beta_{jp}^{(k)}(t)(1-\gamma_{kp}(t)) = 1}.
\end{align}
If client~$j$ does not request content from server~$p$ via any AN, i.e., $\sum_{k \in [K]} \beta_{jp}^{(k)}(t)=0$, then $x_{jp}^{(k)}(t)=0,\,\, \forall\, k \in [K]$. 

Substituting \eqref{eq:reward_design} into the age recursion \eqref{eq:REcurAgeClient}, we obtain
\begin{align}\label{eq:hxrelationship}
h_{jp}(t+1) = h_{jp}(t)+1 - \sum_{k\in[K]}x_{jp}^{(k)}(t).
\end{align}
Applying \eqref{eq:hxrelationship} recursively and noting that $h_{jp}(0)=1$, it follows that
\begin{align}\label{eq:hxequ}
h_{jp}(t) = -\sum_{\tau=0}^{t-1} \sum_{k\in[K]}x_{jp}^{(k)}(\tau) +t+1.
\end{align}

As a result, according to \eqref{eq:hxequ},  minimizing the myopic AoI objective in~\eqref{eq:Opt1} is equivalent to maximizing the accumulated slot-based reward. This yields the following equivalent reward maximization problem, solved sequentially over time for $0\le t\le T-1$:
\begin{align}\label{eq:Opt2}
\text{(P3)}\;\; & \max_{\pi} \;\; \frac{1}{TJP}\sum_{\substack{j\in[J]\\p\in[P]}}\E\left[\sum_{k\in[K]}x_{jp}^{(k)}(t)\right].\nonumber\\
&\;\;\text{s.t.}\;\; \eqref{eq:RqstOneProb}
\end{align}

This reformulation transforms the original age minimization problem~\eqref{eq:Opt0} into a slot-based reward maximization problem.

\subsection{Non-Stationary MAB and AoI Regrets}\label{subsec:AgeRgrts}

The optimization in \eqref{eq:Opt2} naturally aligns with decentralized online learning and lends itself to a MAB formulation. In particular, for each client-server pair $(j, p)$, the ANs can be viewed as arms, and each client independently interacts with the environment to balance \textit{exploration} (discovering better ANs) and \textit{exploitation} (selecting known to yield higher rewards). While MAB provides a lightweight framework for decentralized online learning, classical MAB models still remain invalid for our system.

The first reason is the \textit{dependence among rewards observed by different clients}. Unlike collision-based decentralized MAB models~\cite{rosenski2016multi,landgren2021distributed}, where multiple agents selecting the same arm independently receive zero rewards, our system is fundamentally different. Specifically, when multiple clients simultaneously request content from the same server~$p$ via a common AN, the AN makes a single update decision---either fetching a fresh packet from the server or serving cached content. This single update decision affects all requesting clients and further determines their received rewards. As a result, the rewards observed by different clients selecting the same AN are no longer independent. This structural dependence violates the reward independence assumptions commonly adopted in classical multi-armed bandit models. 

The second reason is \textit{non-stationarity}. Unlike classical stationary bandit problems, the reward distributions in our system evolve over time and depend on history-dependent AoI states.  For example, if client~$j$ selects an AN whose cached packet for server~$p$ has not been updated for a long period and the AN decides to serve cached content at time~$t$, the resulting AoI reduction---and hence the reward---will be small. In contrast, if client~$j$ currently has a large AoI for content~$p$ and selects an AN that has recently fetched a fresh update from server~$p$, the reward obtained in that slot can be significantly larger.  This complex reward structure is consistent with the \textit{Aging bandit problem}~\cite{Agingbandit}: the reward depends not only on the current state of the selected AN, but also on the history-dependent AoI evolution. 

These two reasons above fundamentally distinguish our setting from classical bandit models and motivate the need for adaptive learning algorithms capable of tracking non-stationary dynamics. We therefore cast the optimization problem~\eqref{eq:Opt2} as a \textit{decentralized}, \textit{non-stationary} MAB problem. To quantify the performance of learning algorithms in such an environment, we introduce the notion of \textit{AoI regret}~\cite{Agingbandit}.

We use $x_{jp}^{a_{jp,t}}(t)$ to denote the reward obtained by client~$j$ when requesting content from server~$p$ via AN $a_{jp,t}$ at time~$t$. Similarly, let $x_{jp}^{*}(t)$ denote the reward obtained under an oracle optimal policy that selects the best AN for each client---server pair $(j,p)$ at time~$t$. The AoI regret of our policy~$\pi$ after $T$ rounds is then defined as
\begin{align}\label{eq:regret_T}
R^{\pi}(T)
= \sum_{t=1}^{T} \sum_{j=1}^{J} \sum_{p=1}^{P}
\mathbb{E}\!\left[ x_{jp}^{*}(t) - x_{jp}^{a_{jp,t}}(t) \right].
\end{align}

\section{The \MAB Algorithm}\label{sec:ProAlg}

\subsection{Design Principles}\label{subsec:motivation}
To address the challenges mentioned in Sections~\ref{subsec:Challenges} and~\ref{subsec:AgeRgrts}, we propose the \MAB (\mab) algorithm, which consists of three key components:
\begin{compactenum}[(i)]
\item \textbf{Adaptive windowing with reset:} \mab combines the adaptive windowing (ADWIN) technique with a reset mechanism. When ADWIN detects a significant change, the algorithm \textit{discards all outdated statistics} and learns from the new environment~\cite{garivier2008upper,trovo2020sliding}.
\item \textbf{Periodic monitoring mechanism:} To timely track changes in reward dynamics, \mab partitions time into blocks and designates a subset of rounds within each block as monitoring rounds.
During non-monitoring rounds, the algorithm exploits the currently estimated optimal arm, while monitoring rounds are used to periodically assess potential changes in the reward distributions~\cite{JMLR}.
\item \textbf{Detection for abrupt and gradual changes:} \mab extends the single-agent ADR~\cite{JMLR} framework to a \textit{decentralized} multi-agent setting with \textit{correlated} rewards, providing an effective solution for detecting both abrupt and gradual changes.
\end{compactenum}

\subsection{Implementation of the Design Principles}\label{subsec:Algoperation}

\paragraph{Adaptive windowing with reset}\label{subsubsec:adwin}
To detect changes in reward distributions, \mab adopts ADWIN~\cite{bifet2007learning} as shown in Algorithm~\ref{alg:ADWIN}. The key idea of ADWIN is to monitor whether the average reward within a sliding window changes significantly over time.

At each time~$t$, the newly observed reward is appended to the current window, denoted by $W(t+1)$. ADWIN then considers all possible consecutive partitions of $W(t)$ into two sub-windows $W_1$ and $W_2$. 
\begin{definition}\label{defn:Detection}
A change is detected at time $t$ if there exists a consecutive partition $W(t+1) = W_1 \cup W_2$ such that:
\begin{align}\label{eq:ChangeDetect}
|\hat{\mu}^{(k)}_{jp,W_1} - \hat{\mu}^{(k)}_{jp,W_2}| \geq \varepsilon_{\text{cut}}^\delta,
\end{align}
\end{definition}
where $\hat{\mu}^{(k)}_{jp,W_1}$ and $\hat{\mu}^{(k)}_{jp,W_2}$ denote the empirical mean reward when client~$j$ requests content from server~$p$ via AN~$k$ during sub-windows $W_1$ and $W_2$. Moreover, we define $\varepsilon_{\text{cut}}^\delta$ as follows:
\begin{align*}
\varepsilon_{\text{cut}}^\delta = \sqrt{ \frac{1}{2|W_1|} \log \left( \frac{1}{\delta} \right) } + \sqrt{ \frac{1}{2|W_2|} \log \left( \frac{1}{\delta} \right) },\, \delta = \frac{1}{T^3},
\end{align*}
where $\varepsilon_{\text{cut}}^\delta$ is designed based on Hoeffding's inequality~\cite{JMLR}.

Once a change is detected, the algorithm immediately performs a reset, discarding outdated observations and restarting the learning process from time $t$ onward, using only the remaining horizon $T-t$ to adapt to the reward dynamics. 
\begin{remark}
Although clients’ rewards are correlated, the reset mechanism mitigates the influence of such correlation after environmental changes. When a client detects a change and performs a reset, it discards historical observations accumulated under the previous environment. As a result, each client can re-estimate rewards based on fresh observations and adapt more effectively to the new environment.

\end{remark}

\begin{algorithm}[htbp]
\caption{Adaptive Windowing (ADWIN)}
\label{alg:ADWIN}
\begin{algorithmic}[1]
\REQUIRE Reward stream $S = (x_1, x_2, \ldots)$, confidence level $\delta \in (0,1)$
\STATE Initialize window $W(1) = \emptyset$
\FOR{$t = 1,2,\ldots$}
    \STATE $W(t+1) = W(t) \cup \{x_t\}$
    \FOR{every split $W(t+1) = W_1 \cup W_2$}
        \STATE Compute empirical means: $\hat{\mu}_{W_1}$ and $\hat{\mu}_{W_2}$
        \IF{$|\hat{\mu}_{W_1} - \hat{\mu}_{W_2}| \ge \varepsilon_{\text{cut}}^\delta$}
            \STATE \textbf{return True} \quad (change detected)
        \ENDIF
    \ENDFOR
\ENDFOR
\STATE \textbf{return False}
\end{algorithmic}
\end{algorithm}

\paragraph{Periodic Monitoring}\label{subsubsec:monitor}
The \mab algorithm employs a monitoring mechanism that periodically selects specific arms to track reward dynamics. Specifically, the horizon $T$ is partitioned into a sequence of blocks, indexed by $l = 1, 2, \dots, \left\lceil \log\left( \frac{T}{KN} + 1 \right) \right\rceil$. Each block consists of $\mathcal{O}(2^{l-1})$ subblocks, and each subblock spans $KN$ time slots, where $N$ is a monitoring parameter ($K$ is the number of ANs).

Within each subblock, rounds are divided into monitoring and non-monitoring rounds. During non-monitoring rounds, client~$j$ selects AN~$I_{jp}(t)$ to request content from server~$p$ according to the Upper Confidence Bound (UCB) algorithm~\cite{auer2002finite}. Specifically, the selected AN is given by
\begin{equation}
    I_{jp}(t) = \arg\max\limits_{k \in [K]} \left(\hat{\mu}_{jp}^{(k)}+\sqrt{\frac{2\log(t)}{T_{jp}^{(k)}(t)}}\right),
    \label{eq:UCB_selection}
\end{equation}
where $\hat{\mu}_{jp}^{(k)}$ denotes the empirical mean reward, $T_{jp}^{(k)}(t)$ is the number of times AN~$k$ has previously been selected by client~$j$ for content from server~$p$ up to time $t$.

During monitoring rounds, client~$j$ sends a request for server~$p$ to AN~$i_{jp}^{(l-1)}$ to periodically track potential changes in reward distributions. Before the final subblock of block~$l$, the \mab algorithm selects a new monitoring AN~$i_{jp}^{(l)}$ for block~$l+1$ based on the historical selection frequency during the non-monitoring rounds:
\begin{align}\label{eq:monitoring_arm_selection}
i_{jp}^{(l)} = \arg\max_{k \in [K]} N_{jp}^{(k)},
\end{align}
where $N_{jp}^{(k)}$ denotes the number of times client~$j$ has selected AN~$k$ to request content from server~$p$ during the non-monitoring rounds up to the current time:
\begin{align*}
N_{jp}^{(k)} = |\{ s : I_{jp}(s) = k \text{ and s is a non-monitoring round} \}|.
\end{align*}
\begin{remark}
This selection prioritizes the AN with the most observations, ensuring that the monitoring process is based on reliable empirical estimates. Therefore, even if the reward distribution changes slowly, it can be detected through accumulated observations.
\end{remark}

\subsection{Complete Algorithm Description}\label{subsec:algdes}

Algorithm~\ref{alg:ABAR} summarizes the complete \mab procedure, integrating all components described above.

Compared with the ADR framework~\cite{JMLR}, \mab introduces two key extensions. First, the existing ADR framework is built on simplied models that assume changes are either strictly abrupt or strictly gradual. In contrast, in our setting, reward dynamics are history-dependent and evolve through a combination of abrupt and gradual changes. By integrating periodic monitoring with adaptive resets, \mab does not require prior assumptions about change patterns, enabling reliable adaptation to complex, real-world network dynamics. Second, \mab operates in a decentralized multi-agent setting, where each client runs its own instance of the algorithm. While clients act independently, shared observations introduce statistical coupling among agents. By resetting and discarding outdated statistics, \mab alleviates the impact of such coupling and improves adaptability to non-stationary environments.

Together, these extensions enable \mab to maintain reliable performance in decentralized and non-stationary environments.

\begin{algorithm}[htbp]
\caption{\MAB (\mab) for the pair $(j, p)$}
\label{alg:ABAR}
\begin{algorithmic}[1]
\REQUIRE Confidence level $\delta$, monitoring parameter $N \in \mathbb{N}$
\STATE Initialize UCB statistics $\hat{\mu}_{jp}^{(k)}$, $T_{jp}^{(k)}$, $N_{jp}^{(k)}$ for all $k\in[K]$
\FOR{$l = 1$ to $\left\lceil \log_2\left( \frac{T}{KN} + 1 \right) \right\rceil$}
    \FOR{$t = (2^{l-1} - 1)KN + 1$ to $\min\left\{ (2^{l} - 1)KN, T \right\}$}
        \IF{$l \geq 2$ and $t \bmod K=0$}
            \STATE $I_{jp}(t)=i_{jp}^{(l-1)}$ 
            (monitoring AN of previous block)
        \ELSIF{$l \geq 2$ and $t \bmod K=1$ and $t \geq (2^{l} - 2)KN + 1$}
            \STATE $I_{jp}(t)=i_{jp}^{(l)}$ 
            (monitoring AN of current block)
        \ELSE
            \IF{$\sum_{k \in [K]} \beta_{jp}^{(k)}(t) = 1$}
                \STATE Select AN based on~\eqref{eq:UCB_selection} and update $N_{jp}^{(I_{jp}(t))}$
            \ENDIF
        \ENDIF
        \STATE Update AoI according to~\eqref{eq:REcurAgeClient} and~\eqref{eq:RecurAgeeRRH}
        \STATE Update the empirical mean reward based on~\eqref{eq:reward_design}
        \IF{ADWIN detects change for client-server pair $(j,p)$}
            \STATE Reset all statistics: $\hat{\mu}_{jp}^{(k)},T_{jp}^{(k)}(t),N_{jp}^{(k)},\forall k\in[K]$
            \STATE Reset the algorithm with $T \leftarrow T - t$
        \ENDIF
        \IF{$t = KN$ \textbf{or} ($l \geq 2$ \textbf{and} $t = (2^{l} - 2)KN$)}
        \STATE $i_{jp}^{(l)} = \arg\max_{k \in [K]} N_{jp}^{(k)}$

            (select AN for next monitoring phase)
        \ENDIF
    \ENDFOR
\ENDFOR
\end{algorithmic}
\end{algorithm}
At the end of this section, we present a simple observation about Algorithm~\ref{alg:ABAR}. By construction, the current monitoring arm $i^{(l-1)}$ will be periodically selected $N$ times in each subblock; while in the last subblock of the $l$-th block, the algorithm will select a new arm $i^{(l)}$ as the monitoring arm for the next round.
\begin{observation}[Monitoring consistency]\label{obs:Monitoring}
For any block $l=1,2,\dots$, there exists at least one arm that is selected at least $N$ times in each subblock of block $l$.
\end{observation}

\section{Preliminaries: Notations, Definitions, and Assumptions}\label{sec:assumptions}
In this section, we introduce necessary notations, definitions, and assumptions. For clarity, we illustrate these preliminaries in a simplified setting with \textit{a single client, a single server, and multiple ANs}. The analysis framework can be extended to scenarios with multiple clients and multiple servers straightforwardly.

We begin by extending the definitions of gradual and abrupt reward changes introduced in~\cite{JMLR}. Let $\mu_{i,t}$ denote the expected reward of arm $i$ at time slot~$t$. Moreover, we denote $i^{(l)}$ as the monitoring arm selected by the algorithm in block~$l$.

\begin{definition}[Gradual and Abrupt Changes]\label{defn:GradualAbrupt}
Let $b\in(0, 1)$ be a positive scalar and $t\in\N$. Arm~$i$ undergoes a gradual change in time slot $t$ if
\begin{align}\label{eq:GradualChange}
|\mu_{i, t+1} - \mu_{i, t}| \leq b;
\end{align}
and undergoes an abrupt change in time slot $t$ if
\begin{align}\label{eq:AbruptChange}
|\mu_{i, t+1} - \mu_{i, t}| > b.
\end{align}
\end{definition}

\begin{definition}[Change Points]\label{defn:ChangePoints}
Let $b$ be given in Definition~\ref{defn:GradualAbrupt}. Time $t$ is called a change point if there exists $i\in[K]$ such that
\begin{align}\label{eq:ChangePoint}
\,|\mu_{i, t+1}-\mu_{i, t}|>b.
\end{align}
\end{definition}

\begin{definition}[Gradual Segment]\label{defn:GradualSegment}
A gradual segment with respect to arm~$i$ is a maximal consecutive sequence of time slots in which the gradual condition~\eqref{eq:GradualChange} holds.
\end{definition}
According to Definition~\ref{defn:GradualSegment}, an abrupt change at time $t$ disrupts the ongoing gradual segment and re-starts a new segment beginning at $\mu_{i, t}$.

\begin{assumption}\label{assu:MChangePoints}
Within the time interval $[0, T]$, we assume that the system undergoes $M$ change points, whose occurrence times $(T_1,\cdots,T_M)$ are mutually independent random variables. The set of these change points is denoted by 
\begin{align}\label{eq:ChangePoints}
\cT_{c} = \set{T_{1}, T_{2}, \dots, T_{M}}.
\end{align}
For notational convenience, we denote $T_{0}=0$ and $T_{M+1}=T$.
\end{assumption}
\begin{definition}\label{defn:NonEmptyM}
For any $1\le m\le M$, define 
\begin{align}\label{eq:NonEmptyM}
\cK_m = \set{i\,\middle|\,|\mu_{i, T_{m}+1}-\mu_{i, T_m}|>b, i\in[K]}.
\end{align}
\end{definition}
As defined in Definition~\ref{defn:NonEmptyM}, $\cK_m$ denotes the set of arms that satisfy condition~\eqref{eq:AbruptChange} at time $T_m$. By Definition~\ref{defn:ChangePoints}, this set is nonempty for every change point, i.e., $\cK_m \neq \emptyset$.

\begin{assumption}[{\cite[Definition~15]{JMLR}}]\label{assu:monitoring_ChangePoints}
We assume that for each change point, there exists an arm $j\in\set{i^{(l)},i^{(l-1)}}$ such that condition~\eqref{eq:AbruptChange} is satisfied.
\end{assumption}

\begin{assumption}\label{assu:AbruptReset}
We assume that each abrupt change triggers a detection, as specified in Definition~\ref{defn:Detection}.
\end{assumption}

This assumption is justified by Lemma~\ref{lem:Detection}, which shows that the \mab algorithm detects abrupt changes with high probability within a bounded delay. It is also standard in prior work (see~\cite{JMLR}) and aligns naturally with the operational logic of the \mab. Empirical evaluations further confirm that the algorithm reacts reliably to abrupt changes in practice.

\begin{definition}[Resets]\label{defn:Resets}
Suppose Assumption~\ref{assu:AbruptReset} holds. A reset that follows a detection triggered by an abrupt change is called an \textit{abrupt reset}, while any other reset is referred to as a \textit{gradual reset}.
\end{definition}

\begin{definition}[Reset Times]\label{defn:ResetTime}
Let abrupt and gradual resets be defined in Definition~\ref{defn:Resets}, we define  
\begin{compactenum}[(i)]
\item $X_{t}$ as the time of the most recent gradual reset {\it strictly} before time $t$, with $X_{t} = 0$ if no such reset has occurred;
\item $Y_{t}$ as the time of the most recent abrupt reset {\it strictly} before time $t$, with $Y_{t} = 0$ if no such reset has occurred.
\end{compactenum}
\end{definition}

\begin{definition}[Drift-Tolerant Regret, Definition~12 in \cite{JMLR}]\label{defn:Drift-tolerant regret}
Assume a non-stationary environment that is abruptly or gradually changing. Let
\begin{align}\label{eq:Deltai}
\Delta_{i}=\max_{j}\mu_{j,1}-\mu_{i,1}.
\end{align}
be the gap at $t=1$, and 
\begin{align}\label{eq:Epsilont}
\epsilon(t)=\max_{s\leq t}\max_{i}|\mu_{i,s}-\mu_{i,1}|
\end{align}
be the maximum drift of the arms by time step $t$. For $c>0$, let
\begin{align}\label{eq:Reg}
\mathrm{Reg}_{\mathrm{tr}}(T,c):=\sum_{t=1}^T\left(\mathrm{reg}(t)-c\cdot\epsilon(t)\right)^{+}
\end{align}
where $(x)_{+}=\max(x,0)$. A bandit algorithm has logarithmic drift-tolerant regret if a factor $c_{\text{dt}}=O(1)$ exists such that
\begin{align}\label{eq:TDR}
\mathbb{E}[\mathrm{Reg}_{\mathrm{tr}}(T,c_{\text{dt}})]\leq c_{\text{dt}}\sum_{\Delta_{i}>0}\frac{\log T}{\Delta_{i}}.
\end{align}
\end{definition}
\begin{remark}
We introduce the notion of Drift-tolerant Regret to avoid penalizing errors that are inherently caused by environmental non-stationarity.
\end{remark}
Note that the mean reward $\mu_{i, t}$ evolves over time, while the algorithm can only form estimates real time based on past observations. As a result, some level of estimation error is unavoidable in non-stationary environment. Motivating by this fact, the idea behind Definition~\ref{defn:Drift-tolerant regret} is to distinguish between \textit{natural errors} induced by the drift of the mean rewards and \textit{excess errors} attributable to algorithmic inefficiency. Specifically, at time $t$, if the instantaneous regret is below a threshold $c\epsilon(t)$, this portion is regarded as a natural error and excluded from the cumulative regret. Only the regret exceeding $c\epsilon(t)$ is accumulated.

When $\epsilon(t)=0$, the environment is stationary, and the Drift-tolerant regret reduces to the standard definition in~\cite[\textit{Definition}~11]{JMLR}.

\begin{assumption}\label{assu:Drift-tolerant regret}
 We assume that the base-bandit of our algorithm (i.e., UCB) has logarithmic drift-tolerant regret.
\end{assumption}

\begin{remark}\label{remark:upperbound of Reg(s)}
Under Assumption~\ref{assu:Drift-tolerant regret}, suppose no reset occurs before time slot $S$. Then there exists a constant $c_{\text{dt}}=O(1)$, such that the cumulative regret up to $S$ satisfies
    \begin{align*}
        \mathbb{E}[\mathrm{Reg}(S)]\leq c_{\text{dt}}\left(\sum_{\Delta_{i}>0}\frac{\log T}{\Delta_{i}}+\mathbb{E}[\sum_{i=1}^{S}\epsilon(t)]\right),
    \end{align*}
    with a similar proof of {\cite[Lemma~17]{JMLR}}.
\end{remark}

\begin{definition}[Detectability]\label{defn:Detectability}
Suppose Assumption~\ref{assu:AbruptReset} holds, and let $\cK_m$ be as in Definition~\ref{defn:NonEmptyM}. For the $m$-th change point, define
\begin{align}\label{eq:Detectability}
\epsilon_{m} = \min_{i\in\cK_m}|\mu_{i,T_{m}} - \mu_{i,T_{m}+1}|.
\end{align}
We say that the $m$-th change point is detectable if the following two conditions hold:
\begin{compactenum}[(i)]
\item $\epsilon_{m}\geq \sqrt{\frac{\log (T^{3})}{2U_{m}}}+6bKN+2\sqrt{\frac{\log(T^3)}{2N}}+b$.
\item $T_{m} - X_{T_m} \geq 32KU_{m}$.
\end{compactenum}
\end{definition}

Definition~\ref{defn:Detectability} is different from the counterpart \cite[Definition~20]{JMLR}. In \cite{JMLR}, the reward is assumed to be stationary between change points, our setting permits gradual changes over time. As such, we introduce a modified notion of detectability tailored to this scenario.

\begin{assumption}\label{assu:epsilon_upper_bound}
For each $m\in\{1,2,...,M\}$, assume that 
\begin{align*}
\epsilon_{m}\leq c_{u}(\sqrt{\frac{\log (T^{3})}{2U_{m}}}+6bKN+2\sqrt{\frac{\log(T^3)}{2N}}+b),
\end{align*} 
where $c_{u}$ is a constant.
\end{assumption}

According to Remark~\ref{remark:Hoeffding's inequality} in Appendix~\ref{Appe:LemmaHoeffding}, we know that $\epsilon_{m}$ will have a corresponding upper bound. To facilitate the derivation of Theorem~\ref{th:abrupt}, we present Assumption~\ref{assu:epsilon_upper_bound}.

\section{Asymptotic Optimality}\label{sec:ProAnl}

In this section, we present rigorous theoretical results characterizing the regret of the proposed algorithm. For clarity, we illustrate the results in a simplified setting with \textit{a single client, a single server, and multiple ANs}. Extensions to multiple clients and multiple servers follow naturally.

We divide the entire time horizon $[0, T]$ into 
\begin{align*}
&[0, X_{T_{1}}],\,\, \set{(X_{T_{m}}, Y_{T_{m+1}}]}_{m=1}^{M},\\
&\set{(Y_{T_{m+1}}, X_{T_{m+1}}]}_{m=1}^{M},\,\, \text{and} \,\, (X_{T_{M+1}}, T].
\end{align*}

Specifically, the intervals 
\begin{align*}
\set{(X_{T_{m}},Y_{T_{m+1}}]}_{m=1}^{M}
\end{align*}
correspond to \textit{abrupt reset intervals}, during which the environment has already changed but the algorithm has not yet detected the change. We denote the union of these intervals by $T_{\text{abrupt}}$.

The remaining intervals,
\begin{align*}
[0, X_{T_1}],\,\, \{(Y_{T_{m+1}},\, X_{T_{m+1}}]\}_{m=1}^M,\,\, \text{and}\,\, (X_{T_{M+1}}, T],
\end{align*}
correspond to \textit{gradual reset intervals}, where changes accumulate gradually and resets are triggered due to the accumulated drift. We denote these intervals by $T_{\text{gradual}}$.

We decompose the total regret into two components: the regret incurred during abrupt reset intervals, and the regret accumulated during gradual reset intervals:
\begin{align*}
\mathbb{E}[\text{Reg}(T)] \triangleq \mathbb{E}[\text{Reg}(T_{\text{abrupt}})] + \mathbb{E}[\text{Reg}(T_{\text{gradual}})]. 
\end{align*}
Let the instantaneous regret at time $t$ be defined as
\begin{align}\label{eq:Reg}
\text{Reg}(t)\triangleq\max_i\mu_{i,t}-\mu_{I(t),t},
\end{align}
where $\max_{i}\mu_{i,t}$ is the expected reward of the optimal arm at time $t$, and $\mu_{I(t),t}$ is the expected reward of the arm selected by the algorithm at time $t$.
Since $Y_{T_1}=0$, then the two regret components are then given by:
\begin{align*}
\mathbb{E}[\text{Reg}(T_{\text{abrupt}})] =& \mathbb{E} \Big[ \sum_{m=1}^{M}\sum_{t = X_{T_{m}} + 1}^{Y_{T_{m+1}}} \text{Reg}(t)\Big].\\
\mathbb{E}[\text{Reg}(T_{\text{gradual}})]=&\mathbb{E} \Big[\sum_{m=1}^{M+1} \sum_{t = Y_{T_{m}}}^{X_{T_{m}}} \text{Reg}(t) +\sum_{t = X_{T_{M+1}}}^{T} \text{Reg}(t)\Big].
\end{align*}

\begin{theorem}[Regret bound within abrupt reset intervals]\label{th:abrupt}
Suppose that Assumptions~\ref{assu:MChangePoints},~\ref{assu:monitoring_ChangePoints},~\ref{assu:AbruptReset},~\ref{assu:Drift-tolerant regret} and~\ref{assu:epsilon_upper_bound} hold. Assume that $\cT_c$ is a global change with constant $c_a$
(Definition~\ref{Defn:GlobalChange}). Let $\delta=\frac{1}{T^3}$ and choose parameters such that, for all $m$, $N\ge 16U_{m}$, $\frac{T_{m}-X_{T_{m}}}{2}\ge KN$, $N=\mathcal{O}((bK)^{-\frac{2}{3}})$, and $b=T^{-d}(d>0)$. Then, the expected regret accumulated over the abrupt reset intervals
satisfies
\begin{align}
&\mathbb{E}[\text{Reg}(T_{\text{abrupt}})] <  \mathcal{O}(\sqrt{T\log T})+\mathcal{O}(T^{1-\frac{d}{3}} (\log T)^{\frac{3}{2}}).
\end{align}
\end{theorem}

\begin{proof}
\textit{Roadmap}.
\begin{compactenum}[(i)]
\item Under the high-probability event $\cV$ defined in Lemma~\ref{lem:Detection}, the algorithm resets within $16K U_m$ steps after each changepoint $T_m$. We accordingly decompose the regret into two parts: the regret incurred under $\cV^c$ and that under $\cV$. By Remark~\ref{remark:Hoeffding's inequality}, the regret contribution from $\cV^c$ is bounded by $\mathcal{O}(1)$.
\item Conditioning on the event $\cV$, we split the interval $[X_{T_m}+1, Y_{T_{m+1}}]$ at the changepoint $T_m$. Lemma~\ref{lem:regret_gap_relation1} relates the instantaneous regret $\text{Reg}(t)$ to the gap $\Delta_{i,m}^{(1)}$. Combining this relation with the definition of drift-tolerant regret, Jensen’s inequality, and the Cauchy--Schwarz inequality yields an upper bound on the regret accumulated over $[X_{T_m}+1, T_m]$.
\item A similar analysis applies to the interval $[T_m, T_m + 16K U_m]$. Summing over all changepoints and applying the Cauchy--Schwarz inequality leads to the desired bound on the regret accumulated over the abrupt reset intervals.
\end{compactenum}
\end{proof}

\begin{theorem}[Regret bound within gradual reset intervals]\label{th:gradual}
Suppose that Assumptions~\ref{assu:monitoring_ChangePoints} and~\ref{assu:Drift-tolerant regret} hold. Let $\delta = \frac{1}{T^3}$, $b = T^{-d}$ for some $d>0$, and
$N = \mathcal{O}((bK)^{-\frac{2}{3}})$. Then, the expected regret incurred during the gradual reset intervals satisfies
\begin{align*}
&\mathbb{E}[\text{Reg}(T_{\text{gradual}})]<\mathcal{O}\left(\sqrt{(\log T)^{\frac{2}{3}}T^{2-\frac{2}{3}d}+T\log T}\right) \nonumber \\
&+\mathcal{O}\Big(T^{1-\frac{d}{3}}(\log T)^{\frac{3}{2}}\Big).
\end{align*}
\end{theorem}
\begin{proof}
\textit{Roadmap}.
\begin{compactenum}[(i)]
\item We introduce two key events: $\cZ$, under which the drift is bounded as in Lemma~\ref{lem:diffenrence}, and $\mathcal{Y}^{c}$, under which the number of resets is bounded as in Lemma~\ref{lem:reset_number}. The total regret is decomposed into contributions from $\mathcal{Z}\cap\mathcal{Y}^{c}$ and its complement $\mathcal{Z}^{c}\cup\mathcal{Y}$.

\item By Remark~\ref{remark:Hoeffding's inequality} and the definition of $F_{1}$ (see~\eqref{eq:F_1} in APPENDIX~\ref{Appe:reset_number}), the regret incurred under the event $\mathcal{Z}^{c}\cup\mathcal{Y}$ is bounded by $\mathcal{O}\left((\log T)^{\frac{1}{3}}\right)$.

\item Conditioning on $\mathcal{Z}\cap\mathcal{Y}^{c}$, each gradual segment is partitioned into sub-intervals of length at least $F_{1}b^{-\frac{2}{3}}$. For each sub-interval, we establish a relationship between the instantaneous regret $\text{Reg}(t)$ and the gap $\Delta_{i,m,n}^{(3)}$. Combining this with the definition of drift-‑tolerant regret, together with Jensen’s inequality, and the Cauchy--Schwarz inequality yields an upper bound on the regret incurred over the interval $[Y_{T_{m}},X_{T_{m}}]$.

\item Summing over all gradual segments and applying the Cauchy--Schwarz inequality completes the bound on the regret under $\mathcal{Z}\cap\mathcal{Y}^{c}$. Together with the contribution from $\mathcal{Z}^{c}\cup\mathcal{Y}$, this yields the desired regret bound over the gradual reset intervals.
    
\end{compactenum}

\end{proof}

\begin{remark}\label{remark:Suboptimal}
Combining Theorem~\ref{th:abrupt} with Theorem~\ref{th:gradual}, we obtain that the regret of our algorithm grows sublinearly with the time horizon~$T$, which implies the algorithm is asymptotically optimal.
\end{remark}

\section{Numerical Results}\label{sec:NumRes}
In this section, we evaluate the performance of the proposed \mab algorithm in terms of two key metrics: the average AoI defined in~\eqref{eq:Avg_Age_CU} and the cumulative AoI regret defined in~\eqref{eq:regret_T}.
\subsection{Simulation Setup and Parameter Configuration}
We configure the simulation parameters as follows. The time horizon spans $T=6\times10^{5}$ time slots. The network consists of $J=2$ clients, $K=3$ ANs, and $P=1$ server. Each client sends exactly one request per time slot, i.e., $\sum_{k\in[K]} b_{jp}^{(k)}(t)=1,\,\, \forall\, j \in [J],\,p \in [P],\, t \in [T]$. To evaluate algorithm robustness under varying network conditions, we consider two different sets of probabilities that the ANs fetch a fresh packet from the server:
\begin{compactenum}[(i)]
\item \textbf{Scenario 1}: $\{r_{11}=0.1,\ r_{21}=0.4,\ r_{31}=0.7\}$,
\item \textbf{Scenario 2}: $\{r_{11}=0.3,\ r_{21}=0.4,\ r_{31}=0.5\}$.   
\end{compactenum}

In Scenario $1$, the ANs have well-separated update probabilities, making the optimal AN relatively easy to identify. In contrast, Scenario $2$ has closely updated probabilities, so distinguishing the optimal AN becomes more challenging.

\subsection{Benchmark Policies}
To provide performance benchmarks, we compare \mab with several representative baseline policies:
\begin{compactenum}[(i)]
    \item \textbf{D-UCB and SW-UCB}: Classic bandit algorithms designed for non-stationary environments and adapted for decentralized decision-making \cite{garivier2008upper}.
    \item \textbf{M-D-MAMAB}: A decentralized multi-agent bandit algorithm originally proposed for caching applications \cite{8964583}.
    \item \textbf{centralized policy (Oracle)}: An ideal benchmark where a central controller has full knowledge of the expected rewards and always selects the AN with the highest expected reward at each time slot. Thus this policy provides a lower bound on achievable AoI performance.
\end{compactenum}

Note that many existing AoI-based bandit algorithms~\cite{Agingbandit,9559999,9668362} constrain rewards to be bounded in the interval $[0,1]$, which is incompatible with our setting, where rewards defined by AoI reduction are unbounded and history-dependent. Furthermore, since \mab can be viewed as a generalization of the ADR framework~\cite{JMLR} to decentralized environments with AoI-based rewards, ADR is therefore not included as a separate benchmark.

\subsection{Average AoI Performance}

\begin{figure}[htbp]
    \centering
    \begin{subfigure}[b]{0.24\textwidth}
        \centering
        \includegraphics[width=\linewidth]{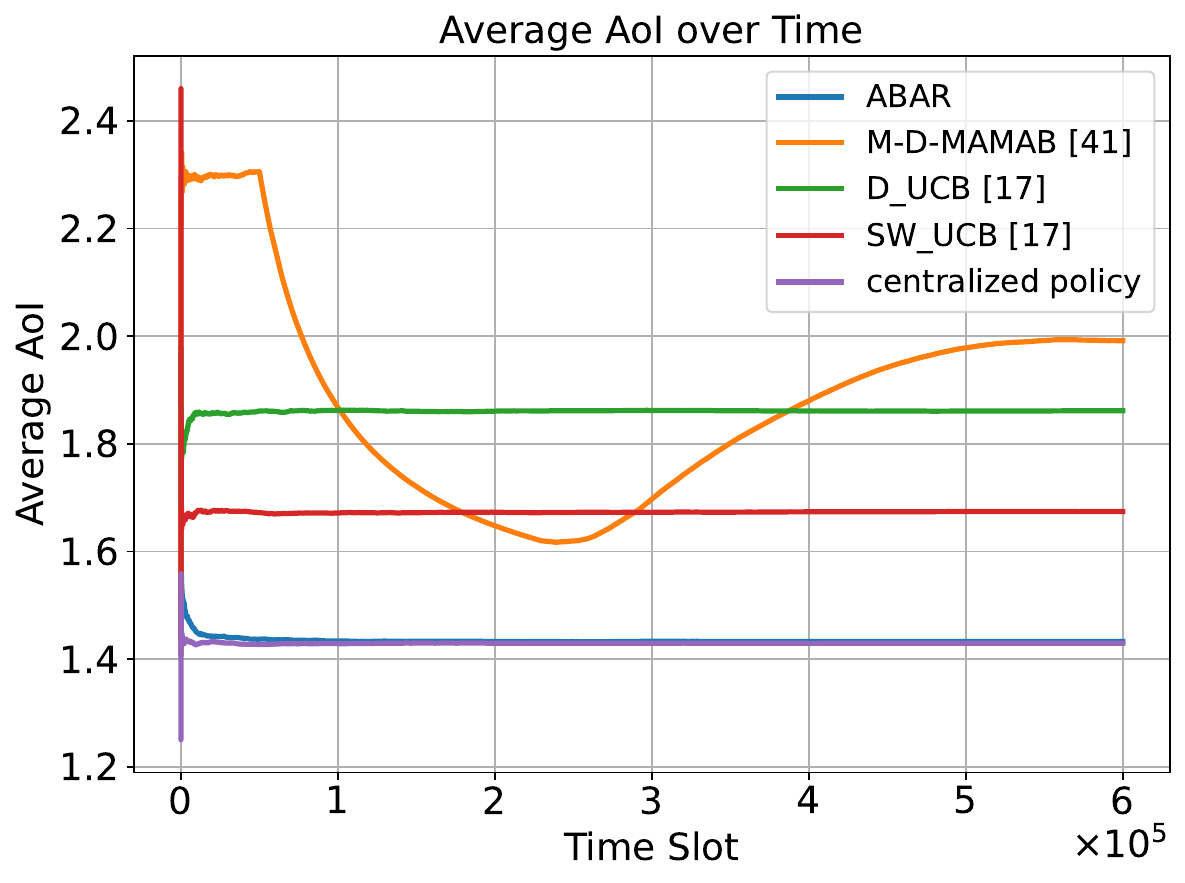}
        \caption{Scenario $1$}
        \label{fig:AoI1}
    \end{subfigure}
    \hfill
    \begin{subfigure}[b]{0.24\textwidth}
        \centering
        \includegraphics[width=\linewidth]{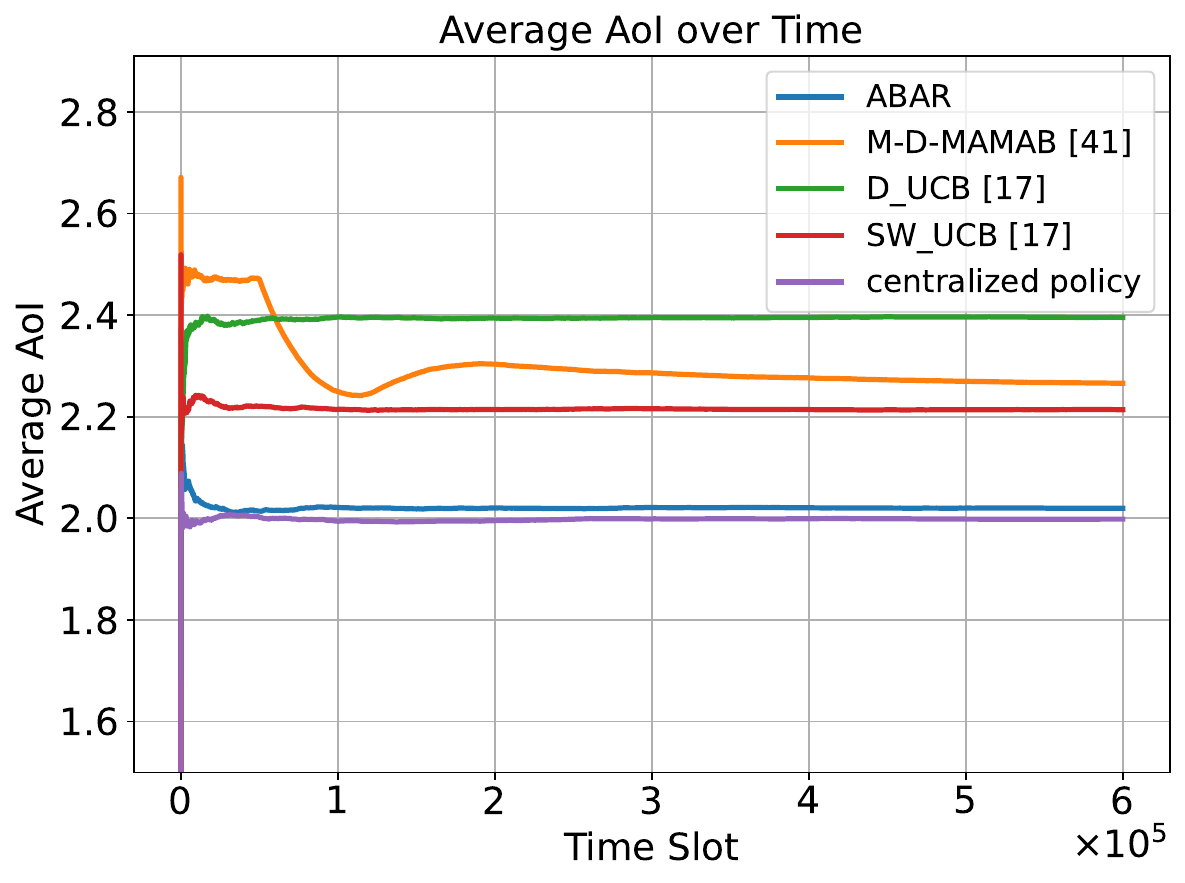}
        \caption{Scenario $2$}
        \label{fig:AoI2}
    \end{subfigure}
    \caption{Average AoI Performance Comparison}
    \label{fig:AoI}
\end{figure}

Fig.~\ref{fig:AoI} compares the evolution of the average AoI under different learning policies for two network scenarios.

As shown in Fig.~\ref{fig:AoI1}, where update probabilities of ANs are well separated, \mab rapidly converges to a low steady-state AoI that is very close to the performance centralized oracle benchmark. This result demonstrates that \mab is able to learn a near-optimal collaborative requesting policy despite operating in a fully decentralized setting and without prior knowledge of the reward statistics.

In contrast, both D-UCB and SW-UCB converge to significantly higher average AoI levels. This performance gap arises because these algorithms are not specifically designed to handle history-dependent and non-stationary AoI-based rewards. The M-D-MAMAB algorithm performs the worst, exhibiting large fluctuations and the highest average AoI. This suggests that although M-D-MAMAB supports decentralized learning, it is less effective at capturing the reward dynamics in our setting.

Fig.~\ref{fig:AoI2} illustrates the average AoI performance in Scenario $2$. In this case, distinguishing the optimal AN becomes more challenging due to the smaller differences in update probabilities. Nevertheless, \mab consistently achieves the lowest average AoI among all decentralized algorithms and remains close to that of centralized oracle. Compared with Scenario~1, the performance gap between \mab and the centralized oracle slightly increases, reflecting the greater learning difficulty in this scenario. Moreover, D-UCB and SW-UCB still converge to substantially higher average AoI levels. Notably, M-D-MAMAB exhibits pronounced instability and slower convergence speed in Scenario $2$.

\subsection{Cumulative AoI Regret Performance}

\begin{figure}[htbp]
    \centering
    \begin{subfigure}[b]{0.24\textwidth}
        \centering
        \includegraphics[width=\linewidth]{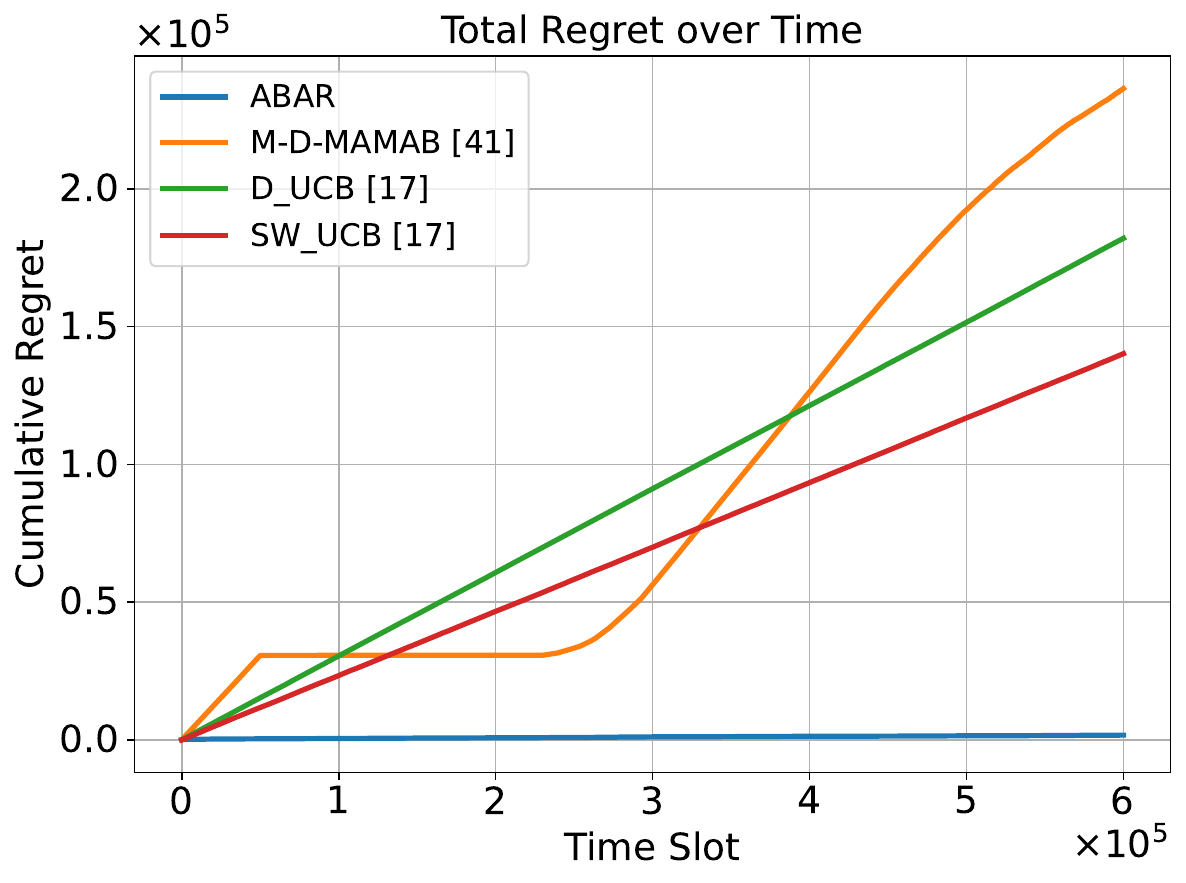}
        \caption{Scenario $1$}
        \label{fig:regret1}
    \end{subfigure}
    \hfill
    \begin{subfigure}[b]{0.24\textwidth}
        \centering
        \includegraphics[width=\linewidth]{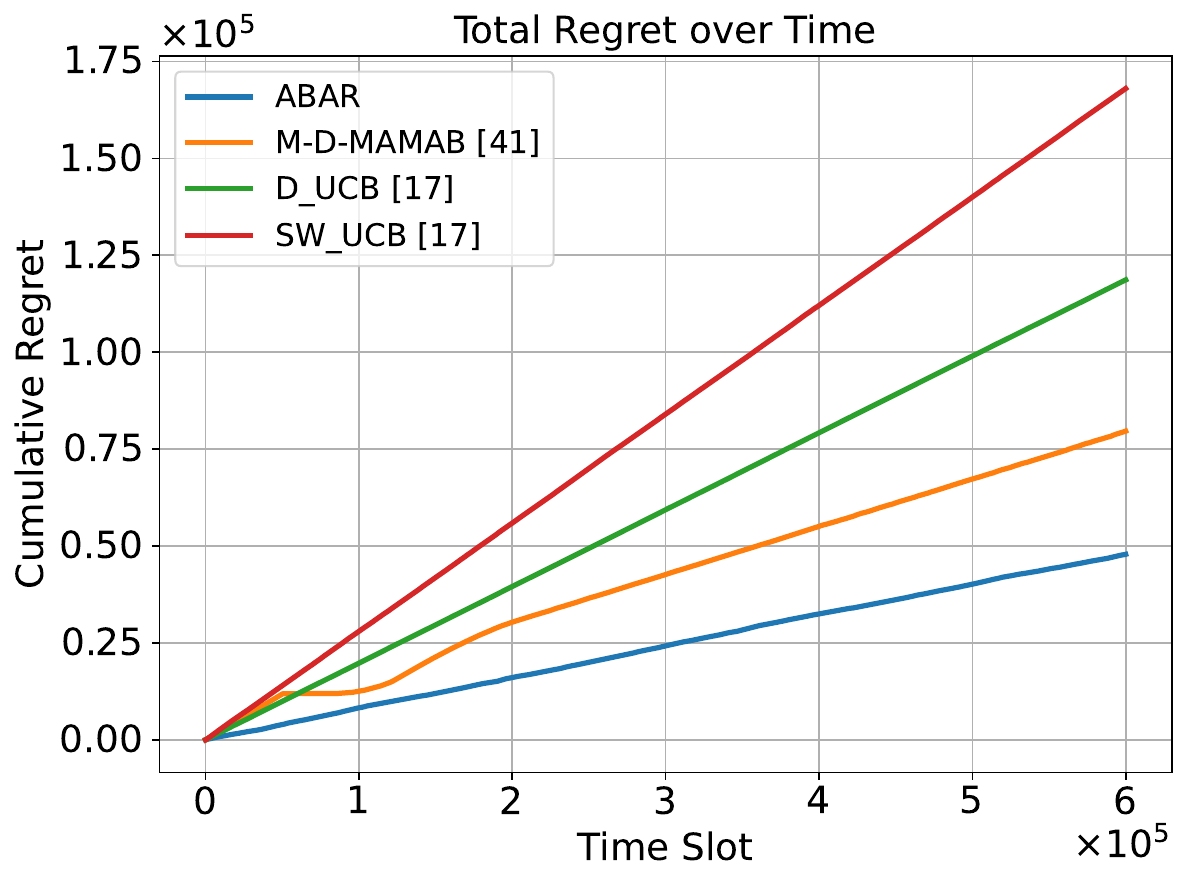}
        \caption{Scenario $2$}
        \label{fig:regret2}
    \end{subfigure}
    \caption{Cumulative AoI Regret Performance Comparison}
    \label{fig:regret}
\end{figure}

Fig.~\ref{fig:regret} illustrates the cumulative AoI regret over time, where a lower regret indicates more efficient learning. The centralized policy always selects the optimal AN and therefore achieves zero regret; it is therefore omitted from the figures.

As shown in Fig.~\ref{fig:regret1},  \mab exhibits the slowest regret growth rate in Scenario~$1$, with the cumulative regret remaining significantly small over the entire time horizon. This indicates that \mab can quickly adapt to non-stationary and history-dependent reward distributions while maintaining near-optimal performance.

In contrast, both D-UCB and SW-UCB display approximately linear regret growth, reflecting their limited ability to track evolving reward statistics. The M-D-MAMAB algorithm exhibits unstable regret behavior, characterized by with multiple changes in slope. This unstable regret growth suggests that its exploration-exploitation mechanism is not well aligned with the AoI-based reward structure in our setting.

Fig.~\ref{fig:regret2} illustrates the cumulative AoI regret in Scenario~$2$. Compared with Scenario~$1$, \mab continues to achieve the best regret performance among all decentralized algorithms, although its final cumulative regret is slightly higher due to the increased difficulty in distinguishing the optimal AN. 

We observe that the regret of \mab grows approximately linearly rather than sublinearly. This behavior can be attributed to the fact that the sublinear regret guarantee in Remark~\ref{remark:Suboptimal} relies on Assumption~\ref{assu:MChangePoints}, which assumes a limited number of change points. In Scenario~$2$, this assumption is violated, as the system can experience a large number of changes in the effective reward dynamics.

Meanwhile, D-UCB and SW-UCB continue to exhibit approximately linear regret growth, while M-D-MAMAB shows pronounced instability with multiple inflection points in its regret curve.

Overall, these results demonstrate that \mab not only achieves near-optimal long-term average AoI but also substantially reduces cumulative learning regret, confirming its effectiveness in decentralized AoI optimization problems under non-stationary and history-dependent reward dynamics.

\section{Conclusion}\label{sec:Conclusion}
In this work, we study a decentralized collaborative requesting problem aimed at minimizing the long-term average AoI in edge networks composed of multiple clients, ANs and servers, where the states of ANs are unknown to the clients. By defining the reward as the AoI reduction, we formulate this sequential decision-making task under the Aging  bandit framework. The reward process is history-dependent and influenced by the actions of other agents, exhibiting both abrupt and gradual changes in epected rewards and resulting non-stationary dynamics. 

To address these challenges, we propose the \mab algorithm. By combining adaptive windowing with periodic monitoring, \mab effectively detect changes in reward distributions and promptly discards outdated observations through reset operations. Compared with existing ADR-based framework, \mab extends the theoretical framework to more general non-stationary setting with mixed change dynamics. We further establish theoretical performance guarantees for \mab and validate its effectiveness through extensive simulations.

Several directions remain for future work: (i) extending the model to combine both content caching and service caching for joint optimization; (ii) taking task deadlines into account to better reflect the time-sensitive requirements in decentralized edge networks.

\bibliographystyle{IEEEtran}
\bibliography{references}

\clearpage

\appendices

\section{Proof of Lemma~\ref{lem:Hoeffding's inequality}}\label{Appe:LemmaHoeffding}

\begin{lemma}[Hoeffding's inequality] \label{lem:Hoeffding's inequality}
Let $p > 0$ be arbitrary and 
\begin{align}\label{eq:cX}
\cX\triangleq\bigcap_{i \in [K]}\bigcap_{W^{\prime} \in \mathcal{W}}\set{\left|\hat{\mu}_{i,W^{\prime}} - \mu_{i,W^{\prime}}\right| \le \sqrt{\frac{\log\left(T^{2+p}\right)}{2\left|(W^{\prime})^{i}\right|}}},
\end{align}
then 
\begin{align}\label{eq:ProbcX}
\Prb\left(\cX\right) \ge 1 - \frac{2K}{T^{p}}.
\end{align}
\end{lemma}

\begin{remark}\label{remark:Hoeffding's inequality}
    The above lemma is based on the assumption that $\mu_{i,s} \in [0,1]$ where $s$ is any round of $W(t)$. 
    
    Then we assume that $\mu_{i,s} \in [0,\alpha]$ where $\alpha>0$. For each fixed $W^{\prime}$ and arm $i$, we can use Hoeffding's inequality to control the estimation error:
    
    \begin{align*}
    \Prb\left(\left|\hat{\mu}_{i,W^{\prime}}-\mu_{i,W^{\prime}}\right|\ge\varepsilon\right)\le 2\exp\left(\frac{-2\left|\left(W^{\prime}\right)^{i}\right|\varepsilon^{2}}{\alpha^{2}}\right).
    \end{align*}

    Similarly, we can get the following conclusion:
    \begin{align*}
    \Prb(\cX^c)\le&\Prb\set{\left|\hat{\mu}_{i,W^{\prime}} - \mu_{i,W^{\prime}}\right| > \sqrt{\frac{\log\left(T^{2+p}\right)}{2\left|(W^{\prime})^{i}\right|}}} \\
    &\le 2 \cdot T^{2} \cdot K \cdot \frac{1}{T^{\frac{2+p}{\alpha^{2}}}} = \frac{2K}{T^{\frac{2+p}{\alpha^{2}}-2}}.
    \end{align*}

    Thus we need to guarantee that $\frac{2+p}{\alpha^{2}}-2>0$, i.e. $\alpha<\sqrt{\frac{2+p}{2}}$. In other words, in order for event $\cX$ to be true with high probability, $\mu_{i,s}$ needs to satisfy $\mu_{i,s}<\sqrt{\frac{2+p}{2}}$.
\end{remark}

\begin{proof}
According to \eqref{eq:cX}, by De Morgan’s Laws, for any $i\in[K]$ and $W'\in\cW$, we obtain:
\begin{align}\label{eq:DeMorgan}
\cX^c =& \bigcup_{i \in [K]}\bigcup_{W^{\prime} \in \mathcal{W}}\set{\left|\hat{\mu}_{i,W^{\prime}} - \mu_{i,W^{\prime}}\right| > \sqrt{\frac{\log\left(T^{2+p}\right)}{2\left|(W^{\prime})^{i}\right|}}}\nonumber\\
\supset&\set{\left|\hat{\mu}_{i,W^{\prime}} - \mu_{i,W^{\prime}}\right| > \sqrt{\frac{\log\left(T^{2+p}\right)}{2\left|(W^{\prime})^{i}\right|}}}.
\end{align}

For each fixed $W^{\prime}$ and arm $i$, we can use Hoeffding's inequality to control the estimation error:
\begin{align*}
\Prb\left(\left|\hat{\mu}_{i,W^{\prime}}-\mu_{i,W^{\prime}}\right|\ge\varepsilon\right)\le 2\exp\left(-2\left|\left(W^{\prime}\right)^{i}\right|\varepsilon^{2}\right).
\end{align*}
Let 
\begin{align*}
\varepsilon=\sqrt{\frac{\log(1/\delta)}{2\left|\left(W^{\prime}\right)^{i}\right|}}=\sqrt{\frac{\log\left(T^{2+p}\right)}{2\left|\left(W^{\prime}\right)^{i}\right|}}.
\end{align*}
we obtain:
\begin{align*}
2\exp\left(-2\left|\left(W^{\prime}\right)^{i}\right|\varepsilon^{2}\right) = \frac{2}{T^{2+p}},
\end{align*}

Note that the size of the window set satisfies $|\cW| \le T^{2}$\footnote{If $t=1$, the number of windows is 1; if $t=2$, it is 2; $\cdots$; if $t=T$, it is $T$. Therefore, $|\cW| \le \frac{T(T-1)}{2} \le T^{2}$.}. Thus, by the union bound over all windows and all arms:
\begin{align*}
\Prb(\cX^c)\le&\Prb\set{\left|\hat{\mu}_{i,W^{\prime}} - \mu_{i,W^{\prime}}\right| > \sqrt{\frac{\log\left(T^{2+p}\right)}{2\left|(W^{\prime})^{i}\right|}}} \\
&\le 2 \cdot T^{2} \cdot K \cdot \frac{1}{T^{2+p}} = \frac{2K}{T^{p}},
\end{align*}
which implies 
\begin{align*}
\Prb\left(\cX\right) \ge 1 - \frac{2K}{T^{p}}.
\end{align*}
\end{proof}

\section{Proof of Lemma~\ref{lem:diffenrence}}
\begin{definition}[Globally gradual changes, Assumption~22 in \cite{JMLR}]\label{Defn:GlobalGradualChange}
The environment is globally gradual with constant $c_{g}\in (0, 1]$ if for all $i,j \in [K]$, and any slots $t,s$ that belong to a gradual segment,
\begin{equation}
|\mu_{i,t} - \mu_{i,s}| \ge c_g |\mu_{j,t} - \mu_{j,s}|.    
\end{equation}
\end{definition}

\begin{lemma}\label{lem:diffenrence}
Suppose the environment is globally gradual with constant $c_{g}$ (Definition~\ref{Defn:GlobalGradualChange}).
Then, with probability at least $1 - \frac{2K}{T}$, the following holds: for any round $t \in [T]$, any arm $i \in [K]$,
and any two rounds $s, s' \in W(t)$ with window size $|W(t)| > N$, where $N$ is a system parameter in our algorithm,
\begin{align}\label{eq:Diffmus}
&|\mu_{i,s} - \mu_{i,s'}| \le\nonumber\\
&\frac{\log T}{c_{g}}\left( 2bKN + 8\sqrt{\frac{\log (T^{3})}{2N}} \right) + b\log T .
\end{align}
\end{lemma}

\begin{proof}
\textit{Roadmap}.
\begin{compactenum}[(i)]
    \item Under the assumption that no reset occurs up to block~$l$, Observation~\ref{obs:Monitoring} ensures that each subblock contains at least $N$ samples of some arm $i_{l}$. This allows us to bound 
    $\left|\hat{\mu}_{i_{l},W_{(l,c)}} - \hat{\mu}_{i_{l},W_{:(l,c)}}\right|$ and $\left|\hat{\mu}_{i_{l},W_{(l,1):(l,c)}} - \hat{\mu}_{i_{l},W_{:(l,1)}}\right|$.
    By applying the triangle inequality, we further obtain an upper bound on the difference between the estimated reward of any subblock~$(l,c)$ and those of the first or last subblocks within block~$l$.

    \item Combining these bounds with Hoeffding's inequality and the fact that the expected reward moves by at most $bKN$ within each subblock, we can use a recursive approach to obtain the upper bound of the difference between the expected rewards of any two rounds in the gradual segment.
\end{compactenum}

\subsection{Block and Subblock Decomposition}

In a gradual segment, we divide the rounds into blocks and subblocks. 
For each $l = 1, 2, \dots$, the $l$-th block is partitioned into $2^{\,l-1}$ subblocks. 
We use a tuple $(l,c)$, where $c = 1, 2, \dots, 2^{\,l-1}$, to denote the $c$-th subblock of the $l$-th block. 
Specifically, subblock $W_{(l,c)}$ corresponds to the rounds
\begin{align*}
\left(KN(2^{\,l-1} + c - 2) + 1, \; \dots,\; KN(2^{\,l-1} + c - 1)\right),
\end{align*}
counted after the most recent reset. We write $\underline{t}_{l}$ and $\bar{t}_{l}$ for the first and last rounds of the $l$-th block:
\begin{align*}
\underline{t}_{l} = KN(2^{\,l-1} - 1) + 1, 
\quad 
\bar{t}_{l} = KN(2^{\,l} - 1).
\end{align*}
For convenience, we introduce two aggregate windows:
\begin{compactenum}[(i)]
\item $W_{:(l,c)}$: the union of all subblocks preceding $W_{(l,c)}$ (excluding $W_{(l,c)}$ itself);
\item $W_{(l,c):(l,c')}$ for $c < c'$: the joint window consisting of consecutive subblocks $W_{(l,c)}, W_{(l,c+1)}, \dots, W_{(l,c'-1)}$.
\end{compactenum}

\subsection{Bounding empirical mean differences within a block}

Fix an arbitrary $l \in \N$. Observation~\ref{obs:Monitoring} implies the following: Assume that no reset occurred up to the $l$-th block. There exists an arm that is drawn at least $N$ times for each subblock $c = 1, 2, \dots, 2^{l-1}$ in the $l$-th block. Moreover, this arm is drawn at least $N$ times in the final subblock of the $(l-1)$-th block. Thus, there exists $i_{l}$ such that for any $l \in \N$ and $c \in [2^{l-1}]$,
\begin{align}
\left|\hat{\mu}_{i_{l},W_{(l,c)}} - \hat{\mu}_{i_{l},W_{:(l,c)}}\right| \le& \sqrt{\frac{\log(T^3)}{2|W_{(l,c)}|^{i_{l}}}}+\sqrt{\frac{\log(T^3)}{2|W_{:(l,c)}|^{i_{l}}}}  \nonumber \\
\le& 2\sqrt{\frac{\log(T^{3})}{2N}}\label{Hoeffding_1}
\end{align}
and
\begin{align}
&\left|\hat{\mu}_{i_{l},W_{(l,1):(l,c)}} - \hat{\mu}_{i_{l},W_{:(l,1)}}\right|\nonumber\\
\le& \sqrt{\frac{\log(T^{3})}{2|W_{(l,1):(l,c)}|^{i_{l}}}}+\sqrt{\frac{\log(T^{3})}{2|W_{:(l,1)}|^{i_{l}}}} \nonumber \\
\le& 2\sqrt{\frac{\log(T^{3})}{2N}},\label{Hoeffding_2}
\end{align}
otherwise a reset should occur. Then, for any $l \ge 2$ and $2\le c \le 2^{l-1}$ we have the expression as ~\eqref{eq:upper_bound}.

\begin{figure*}[!t]
\begin{align}
&\left|\hat{\mu}_{i_{l},W_{(l,c)}}-\hat{\mu}_{i_{l},W_{(l,1)}}\right|\le\left|\hat{\mu}_{i_{l},W_{:(l,c)}}-\hat{\mu}_{i_{l},W_{(l,1)}}\right|+\left|\hat{\mu}_{i_{l},W_{(l,c)}}-\hat{\mu}_{i_{l},W_{:(l,c)}}\right|\nonumber\\
&\le\left|\hat{\mu}_{i_{l},W_{:(l,c)}}-\hat{\mu}_{i_{l},W_{(l,1)}}\right|+2\sqrt{\frac{\log\left(T^{3}\right)}{2N}} \quad \text{(by~\eqref{Hoeffding_1})} \nonumber \\
&=\left|\frac{N_{i_{l},W_{(l,1):(l,c)}}\hat{\mu}_{i_{l},W_{(l,1):(l,c)}}+\left(N_{i_{l},W_{:(l,c)}}-N_{i_{l},W_{(l,1):(l,c)}}\right)\hat{\mu}_{i_{l},W_{:(l,1)}}}{N_{i_{l},W_{:(l,c)}}}-\hat{\mu}_{i_{l},W_{(l,1)}}\right|+2\sqrt{\frac{\log\left(T^{3}\right)}{2N}}\nonumber \\
&\le2\sqrt{\frac{\log\left(T^{3}\right)}{2N}}+\left|\frac{N_{i_{l},W_{(l,1):(l,c)}}\left(\hat{\mu}_{i_{l},W_{:(l,1)}}-\hat{\mu}_{i_{l},W_{(l,1):(l,c)}}\right)}{N_{i_{l},W_{:(l,c)}}}\right|\nonumber \\
&+ \left|\frac{N_{i_{l},W_{(l,1):(l,c)}}\hat{\mu}_{i_{l},W_{:(l,1)}}+\left(N_{i_{l},W_{:(l,c)}}-N_{i_{l},W_{(l,1):(l,c)}}\right)\hat{\mu}_{i_{l},W_{:(l,1)}}}{N_{i_{l},W_{:(l,c)}}}-\hat{\mu}_{i_{l},W_{(l,1)}}\right|\quad \text{(Triangle Inequality)} \nonumber \\
&\le\left|\frac{N_{i_{l},W_{(l,1):(l,c)}}\hat{\mu}_{i_{l},W_{:(l,1)}}+\left(N_{i_{l},W_{:(l,c)}}-N_{i_{l},W_{(l,1):(l,c)}}\right)\hat{\mu}_{i_{l},W_{:(l,1)}}}{N_{i_{l},W_{:(l,c)}}}-\hat{\mu}_{i_{l},W_{(l,1)}}\right|+4\sqrt{\frac{\log\left(T^{3}\right)}{2N}} \quad \text{(by~\eqref{Hoeffding_2})}\nonumber \\
&=\left|\hat{\mu}_{i_{l},W_{:(l,1)}}-\hat{\mu}_{i_{l},W_{(l,1)}}\right|+4\sqrt{\frac{\log\left(T^{3}\right)}{2N}}\le 6\sqrt{\frac{\log\left(T^{3}\right)}{2N}}.\quad \text{(by~\eqref{Hoeffding_1})}\label{eq:upper_bound}
\end{align}
\vspace{0.5em} 
\noindent\rule{\textwidth}{0.4pt} 
\end{figure*}

Also, for $c=1$, \eqref{eq:upper_bound} is trivial. For $l=1$, it is also trivial since $c=1$ must hold from
$c\le 2^{l-1}$. By following the same discussion, we also have
\begin{align}
|\hat{\mu}_{i_{l}, W_{(l,c)}}-\hat{\mu}_{i_{l}, W_{(l, 2^{l-1})}}|\le 6\sqrt{\frac{\log\left(T^{3}\right)}{2 N}}
\end{align}

\subsection{Bounding reward differences over the gradual segment}

By Lemma~\ref{lem:Hoeffding's inequality} with $p=1$ we have
\begin{align}\label{mean_estimation}
|\mu_{i_{l}, W_{(l,c)}}-\hat{\mu}_{i_{l}, W_{(l, c)}}|\le \sqrt{\frac{\log\left(T^{3}\right)}{2 N}}
\end{align}
for any $l \in \mathbb{N}$ and $c \in [2^{l-1}]$ with probability at least $1-\frac{2K}{T}$. Using the fact that $\mu_t$ will not move more than $bKN$ within a subblock of size $KN$, we get the following conclusion:
\begin{align}\label{eq:round_in_subblock}
|\mu_{i_{l}, W_{(l,c)}}-\mu_{i_{l}, t}|\leq bKN,\quad t \in W_{(l,c)}.
\end{align}
We let $s$-th round belong to the subblock $W_{(l,c)}$ and $s^{\prime}$-th round belong to the subblock $W_{(l^{\prime},c^{\prime})}$. Here, we assume without loss of generality that $s<s^{\prime}$. From \eqref{eq:upper_bound}, \eqref{mean_estimation}, and \eqref{eq:round_in_subblock}, we have the conclusion shown as in \eqref{eq:anytime_upper_bound}.
\begin{figure*}[!t]
\begin{align}
&|\mu_{i,s}-\mu_{i,s^{\prime}}|\le|\mu_{i,s}-\mu_{i,\bar{t}_{l}}|+|\mu_{i,\bar{t}_{l}}-\mu_{i,\underline{t}_{l+1}}|+|\mu_{i,\underline{t}_{l+1}}-\mu_{i,s^{\prime}}| \le|\mu_{i,s}-\mu_{i,\bar{t}_{l}}|+b+|\mu_{i,\underline{t}_{l+1}}-\mu_{i,s^{\prime}}| \nonumber \\
&\le\frac{1}{c_{g}}|\mu_{i_{l},s}-\mu_{i_{l},\bar{t}_{l}}|+b+|\mu_{i,\underline{t}_{l+1}}-\mu_{i,s^{\prime}}|\quad \text{(Globally Gradual Changes)}\nonumber \\
&\le\frac{1}{c_{g}}(|\mu_{i_{l},s}-\mu_{i_{l},W_{(l,c)}}|+|\mu_{i_{l},W_{(l,c)}}-\mu_{i_{l},W_{(l,2^{l-1})}}|+|\mu_{i_{l},W_{(l,2^{l-1})}}-\mu_{i_{l},\bar{t}_{l}}|)+b+|\mu_{i,\underline{t}_{l+1}}-\mu_{i,s^{\prime}}|  \nonumber \\
&\le\frac{1}{c_{g}}(|\mu_{i_{l},W_{(l,c)}}-\mu_{i_{l},W_{(l,2^{l-1})}}|+2bKN)+b+|\mu_{i,\underline{t}_{l+1}}-\mu_{i,s^{\prime}}| \quad\text{(by~\eqref{eq:round_in_subblock})} \nonumber \\
&\le\frac{1}{c_{g}}(|\mu_{i_{l},W_{(l,c)}}-\hat{\mu}_{i_{l},W_{(l,c)}}|+|\hat{\mu}_{i_{l},W_{(l,c)}}-\hat{\mu}_{i_{l},W_{(l,2^{l}-1)}}|+|\mu_{i_{l},W_{(l,2^{l}-1)}}-\hat{\mu}_{i_{l},W_{(l,2^{l}-1)}}|+2bKN)+b+|\mu_{i,\underline{t}_{l+1}}-\mu_{i,s^{\prime}}|\nonumber \\
&\le\frac{1}{c_{g}}\left(8\sqrt{\frac{\log(T^{3})}{2N}}+2bKN\right)+b+|\mu_{i,\underline{t}_{l+1}}-\mu_{i,s^{\prime}}|\quad\text{(by~\eqref{eq:upper_bound},~\eqref{mean_estimation})}\label{eq:anytime_upper_bound}
\end{align}
\vspace{0.5em} 
\noindent\rule{\textwidth}{0.4pt} 
\end{figure*}

Similar for \eqref{eq:anytime_upper_bound}, we can get the following conclusion:
\begin{align}
            &|\mu_{i,\underline{t}_{l+1}}-\mu_{i,s^{\prime}}|\nonumber \\
            &\leq|\mu_{i,\underline{t}_{l+1}}-\mu_{i,\bar{t}_{l+1}}|+|\mu_{i,\bar{t}_{l+1}}-\mu_{i,\underline{t}_{l+2}}|+|\mu_{i,\underline{t}_{l+2}}-\mu_{i,s^{\prime}}|  \nonumber \\
            &\leq|\mu_{i,\underline{t}_{l+1}}-\mu_{i,\bar{t}_{l+1}}|+b+|\mu_{i,\underline{t}_{l+2}}-\mu_{i,s^{\prime}}|  \nonumber \\
            &\leq\frac{1}{c_{g}}|\mu_{i_{l},\underline{t}_{l+1}}-\mu_{i_{l},\bar{t}_{l+1}}|+b+|\mu_{i,\underline{t}_{l+2}}-\mu_{i,s^{\prime}}|  \nonumber \\
            &\leq\frac{1}{c_{g}}(|\mu_{i_{l},{W}_{(l+1,1)}}-\mu_{i_{l},{W}_{(l+1,2^{l})}}|+2b  KN)+b\nonumber \\
            &+|\mu_{i,\underline{t}_{l+2}}-\mu_{i,s^{\prime}}|  \nonumber \\
            &\leq\frac{1}{c_{{g}}}(8\sqrt{\frac{\log(T^{3})}{2N}}+2bKN)+b+|\mu_{i,\underline{t}_{l+2}}-\mu_{i,s^{\prime}}|
\label{eq:mean_diffrence}
\end{align}
By recursively applying the inequality in~\eqref{eq:mean_diffrence} for indices $l,l+1,l+2,\dots,l^{\prime}$, we have
\begin{align}\label{eq:Diffmusl}
|\mu_{i, s}-\mu_{i, s^{\prime}}|\le& \frac{l^{\prime}-l+1}{c_{g}}(8 \sqrt{\frac{\log (T^{3})}{2 N}}+2bKN)\nonumber\\
+&b(l^{\prime}-l).
\end{align}
Substituting the fact that $l^{\prime}\le \log T$ to \eqref{eq:Diffmusl}, we obtain \eqref{eq:Diffmus}.

In words, the difference between the mean rewards of any two windows within the same gradual segment 
is upper bounded by a term that grows logarithmically with $T$.
\end{proof}

\section{Proof of Lemma~\ref{lem:Detection}}
\begin{lemma}[Detection Times for Change Points]\label{lem:Detection}
Let Assumptions~\ref{assu:MChangePoints} and \ref{assu:monitoring_ChangePoints} hold, and $\mathcal{T}_d = \{Y_{T_{2}},Y_{T_{3}}, \dots, Y_{T_{M}},Y_{T_{M+1}}\}$ denote the set of detection times of change points where $Y_{t}$ be in Definition~\ref{defn:ResetTime}. Let $\cT_c$ be in Assumption~\ref{assu:MChangePoints} and all change points are detectable (Definition~\ref{defn:Detectability}). Define:
\begin{align*}
\cV = \{ \forall m \in [M], \, 0 \le Y_{T_{m+1}} - T_m \le 16K U_{m} \}. 
\end{align*}
Under the conditions that $\delta = \frac{1}{T^3}$, $N \ge 16 U_{m}$,$\frac{T_{m} - X_{T_m}}{2}\ge KN$ holds for all $m$, we have
\begin{align*}
\Prb(\cV) \ge 1-\frac{2K}{T}.
\end{align*}
\end{lemma}
\begin{remark}
Event $\cV$ states that for each changepoint $T_{m} \in \cT_{c}$, there exists a corresponding detection time $Y_{T_{m+1}}$ within $16KU_{m}$ time steps.
\end{remark}
\begin{remark}
From Lemma~\ref{lem:Detection}, each abrupt change triggers a detection. It implies that the number of resets caused by abrupt change is the same as the number of abrupt change.
\end{remark}

\begin{proof}
\textit{Roadmap}.
\begin{compactenum}[(i)]
    \item Assume no detection occurs within $[X_{T_m},\, T_m + 16KU_m]$ and we split this interval into
    $W_1 = [X_{T_m}, T_m]$ and $W_2 = (T_m, T_m+16KU_m]$. By Observation~\ref{obs:Monitoring}, we obtain there exists an arm $i_{l}$ such that $|W_{i_{l},1}|,|W_{i_{l},2}| \ge 16U_{m}$.

    \item Hoeffding’s inequality provides an upper bound for $|\mu_{i_l,W_1}-\hat{\mu}_{i_l,W_1}|$ and $|\mu_{i_l,W_2}-\hat{\mu}_{i_l,W_2}|$. Using the subblock structure of gradual segments, we decompose the expected reward $\mu_{i_l,W_1}$ into contributions from earlier sub‑blocks and the ongoing sub‑block before $T_{m}$. Recursively applying the triangle inequality and based on the fact that no reset has occurred up to subblock $(l,c-1)$ yields upper bounds on $|\mu_{i_{l},T_{m}}-\mu_{i_{l},W_{1}}|$ and $|\mu_{i_{l},T_{m}+1}-\mu_{i_{l},W_{2}}|$.

    \item Combining these bounds with the detectability condition $|\mu_{i_l,T_m}-\mu_{i_l,T_m+1}|\ge \epsilon_m$, we show that $|\hat{\mu}_{i_l,W_1}-\hat{\mu}_{i_l,W_2}| \ge \varepsilon_{\text{cut}}^\delta$, which would trigger a reset at $T_m + 16KU_m$. This contradicts our assumption that no detection occurs within $[X_{T_m},\, T_m + 16KU_m]$.
\end{compactenum}

\subsection{Contradiction setup and interval split}

We complete the proof by contradiction. Since $X_{T_{m}}$ is the most recent reset time before $T_m$, which was triggered by gradual drift. By the definition of $X_{T_m}$, there is no abrupt reset occurs in the interval $[X_{T_m}, T_{m})$. Assume that there is no detection in $[X_{T_{m}}, T_{m}+16KU_{m}]$. Then for a split $W(t)=W_{1} \cup W_{2}=[X_{T_{m}}, T_{m}+16KU_{m}],W_{1}=W(t) \cap [T_{m}],W_2 = W(t) \setminus W_1$, we have 
\begin{align}\label{eq:W1W2}
|W_1| \ge T_{m}-X_{T_{m}},|W_2| \ge 16K U_{m}. 
\end{align} 

According to \eqref{eq:W1W2}, Definition~\ref{defn:Detectability} and assumption of Lemma~\ref{lem:Detection},  $|W_1|$ has the following lower bound:
\begin{align*}
|W_1| \ge T_{m}-X_{T_{m}} \ge 2KN \ge 32KU_{m}>16KU_{m}
\end{align*}
By Observation~\ref{obs:Monitoring} and Assumption~\ref{assu:monitoring_ChangePoints}, there exists an arm $i_{l} \in [K]$ (such as monitoring arm $i^{(l)}$) such that
\begin{align}\label{eq:BoundsforWi12}
|W_{i_{l},1}|,|W_{i_{l},2}| \ge 16U_{m}
\end{align}

\subsection{Hoeffding's bounds on two splits $|W_{i,1}|$ and $|W_{i,2}|$}

According to Lemma~\ref{lem:Hoeffding's inequality}, by Hoeffding's inequality we have
\begin{align}\label{eq:Hoeffding's inequality1}
    |\mu_{i,W_1} - \hat{\mu}_{i, W_1}| \le \sqrt{\frac{\log(T^3)}{2|W_{i,1}|}} ,\quad \forall i \in [K]
\end{align}
\begin{align}\label{eq:Hoeffding's inequality2}
    |\mu_{i,W_2} - \hat{\mu}_{i, W_2}| \le \sqrt{\frac{\log(T^3)}{2|W_{i,2}|}},\quad \forall i \in [K]
\end{align}
for $i \in [K]$ with probability at least $1-\frac{2K}{T}$. 

\subsection{Decomposition of the reward $|\mu_{i_{l},T_{m}}-\mu_{i_{l},W_{1}}|$ under the block structure}

Without loss of generality, let $T_{m}$ belong to the $c$-th subblock of the $l$-th block, denoted by the tuple $(l,c)$ as defined in the proof of Lemma~\ref{lem:diffenrence}, within the current gradual segment. It is also worth noting that the time elapsed since the most recent reset in this gradual segment is given by $T_{m}-X_{T_{m}}$. Let $t_{1}=KN(2^{l-1}+c-2)+1$ and $t_{2}=KN(2^{l-1}+c-2)$, where $t_{1}$ represents the first time step of the tuple $(l,c)$ and $t_{2}$ denotes the last time step of the preceding subblock. We denote by $\mu_{i_l,\, \widetilde{W}_{(l,c)}}$ the expected reward of arm $i_{l}$ over the $c$-th subblock of the $l$-th block before time $T_m$. That is, the average reward of arm $i_{l}$ between $t_{1}$ and $T_{m}$. 

By definition, $\mu_{i_{l},W_{1}}$ represents the expected reward of arm $i_{l}$ from the most recent reset time $X_{T_{m}}$ up to time $T_{m}$. Similarly, $\mu_{i_{l},W_{:(l,c)}}$ denotes the expected reward of arm $i_{l}$ over all time slots preceding the tuple $(l,c)$ within the same gradual segment, and the length of this time interval is $t_{2}$. Hence, the total expected reward accumulated before tuple $(l,c)$ can be expressed as $t_{2}\mu_{i_{l},W_{:(l,c)}}$. On the other hand, $\mu_{i_{l},\widetilde{W}_{(l,c)}}$ corresponds to the average reward within the ongoing subblock $(l,c)$ before time $T_{m}$, which spans $(T_{m}-X_{T_{m}}-t_{2})$ time slots. Therefore, by aggregating these two portions of the time horizon, we obtain
\begin{align}\label{eq:mutildeW}
\mu_{i_{l},W_{1}}=\frac{t_{2}\mu_{i_{l},W_{:(l,c)}}+(T_{m}-X_{T_{m}}-t_{2})\mu_{i_{l},\widetilde{W}_{(l,c)}}}{T_{m}-X_{T_{m}}}.
\end{align}

Since $T_{m}$  belongs to tuple $(l,c)$, for arm $i_{l}$, we have:
\begin{align}\label{eq:Triangle inequality0}
|\mu_{i_{l},T_{m}}-\mu_{i_{l},W_{1}}| \le& |\mu_{i_{l},T_{m}}-\mu_{i_{l},W_{(l,c)}}|\nonumber\\
+&|\mu_{i_{l},W_{(l,c)}}-\mu_{i_{l},W_{1}}| \nonumber \\
\le& bKN+|\mu_{i_{l},W_{(l,c)}}-\mu_{i_{l},W_{1}}|.
\end{align}
Substituting \eqref{eq:mutildeW} into \eqref{eq:Triangle inequality0}, we obtain:
\begin{align}\label{eq:Triangle inequality1}
&|\mu_{i_{l},T_{m}}-\mu_{i_{l},W_{1}}| \le bKN+\nonumber \\
&|\mu_{i_{l},W_{(l,c)}}-\frac{t_{2}\cdot\mu_{i_{l},W_{:(l,c)}}+(T_{m}-X_{T_{m}}-t_{2})\cdot\mu_{i_{l},\widetilde{W}_{(l,c)}}}{T_{m}-X_{T_{m}}}|.
\end{align}
By applying the triangle inequality to the weighted average term in \eqref{eq:Triangle inequality1}, we obtain:
\begin{align}\label{eq:Triangle inequality2}
&|\mu_{i_{l},T_{m}}-\mu_{i_{l},W_{1}}|\le bKN\nonumber\\
&+\frac{t_{2}}{T_{m}-X_{T_{m}}}\cdot |\mu_{i_{l},W_{(l,c)}}-\mu_{i_{l},W_{:(l,c)}}| \nonumber \\
&+\frac{T_{m}-X_{T_{m}}-t_{2}}{T_{m}-X_{T_{m}}}\cdot |\mu_{i_{l},W_{(l,c)}}-\mu_{i_{l},\widetilde{W}_{(l,c)}}|.
\end{align}
Since $\frac{t_{2}}{T_{m}-X_{T_{m}}}, \frac{T_{m}-X_{T_{m}}-t_{2}}{T_{m}-X_{T_{m}}}\le 1$, then \eqref{eq:Triangle inequality2} reduces to
\begin{align}\label{eq:Triangle inequality3}
|\mu_{i_{l},T_{m}}-\mu_{i_{l},W_{1}}|\le& bKN+ |\mu_{i_{l},W_{(l,c)}}-\mu_{i_{l},W_{:(l,c)}}| \nonumber \\
+&|\mu_{i_{l},W_{(l,c)}}-\mu_{i_{l},\widetilde{W}_{(l,c)}}|.
\end{align}

\subsection{Obtaining the upper bound of $|\mu_{i_{l},T_{m}}-\mu_{i_{l},W_{1}}|$ and $|\mu_{i_{l},T_{m}+1}-\mu_{i_{l},W_{2}}|$}

By triangle inequality, we get the following two inequalities: 
\begin{align}\label{eq:DiffinTriangle inequality3}
&|\mu_{i_{l},W_{(l,c)}}-\mu_{i_{l},W_{:(l,c)}}| \nonumber \\
&\le |\mu_{i_{l},W_{(l,c)}}-\mu_{i_{l},W_{(l,c-1)}}|+|\mu_{i_{l},W_{(l,c-1)}}-\mu_{i_{l},W_{:(l,c)}}|,
\end{align}
and
\begin{align}\label{eq:Diff1inTriangle inequality3}
&|\mu_{i_{l},W_{(l,c)}}-\mu_{i_{l},W_{(l,c-1)}}|\le |\mu_{i_{l},W_{(l,c)}}-\mu_{i_{l},t_{1}}|\nonumber\\
&+|\mu_{i_{l},t_{1}}-\mu_{i_{l},t_{2}}|+|\mu_{i_{l},t_{2}}-\mu_{i_{l},W_{(l,c-1)}}|\nonumber \\
&\le 2bKN+b.
\end{align}

Since $\mu_{i_{l},W_{:(l,c)}}$ denotes the expected reward of arm $i_{l}$ over all time slots preceding the tuple $(l,c)$, the total number of such slots is $KN(2^{l-1}+c-2)$. 
Meanwhile, the condition $|W_{1}| \geq T_{m}-X_{T_{m}} \geq 2KN$ ensures that the interval $W_{1}$ covers at least three subblocks 
(starting from $X_{T_{m}}$, which corresponds to the first slot in $|W_{1}|$, and extending to $T_{m}$, corresponding to at least the $(2KN+1)$-th slot). Thus, the expected reward of arm $i_{l}$ over all time slots preceding the tuple $(l,c-1)$ is well-defined, with a corresponding length of $KN(2^{l-1}+c-3)$.
According to the block structure in our algorithm, the tuple $(l,c-1)$ itself spans $KN(2^{l-1}+c-2)$ time slots. 
Following the same reasoning as in~\eqref{eq:mutildeW}, we obtain
\begin{align}\label{eq:Diff2inTriangle inequality3}
&|\mu_{i_{l},W_{(l,c-1)}}-\mu_{i_{l},W_{:(l,c)}}|\nonumber \\
&= |\mu_{i_{l},W_{(l,c-1)}}-\frac{(2^{l-1}+c-3)\cdot \mu_{i_{l},W_{:(l,c-1)}}+\mu_{i_{l},W_{(l,c-1)}}}{2^{l-1}+c-2}|\nonumber \\
&\le \frac{2^{l-1}+c-3}{2^{l-1}+c-2}\cdot|\mu_{i_{l},W_{(l,c-1)}}-\mu_{i_{l},W_{:(l,c-1)}}|
\end{align}

Since no reset occurs up to the $(c-1)$-th subblock of the $l$-th block, and based on Observation~\ref{obs:Monitoring}, we get
\begin{align}\label{eq:noreset inequality}
|\mu_{i_{l},W_{(l,c-1)}}-\mu_{i_{l},W_{:(l,c-1)}}|\le 2\sqrt{\frac{\log(T^3)}{2N}},
\end{align}
otherwise a reset should occur. By~\eqref{eq:DiffinTriangle inequality3},~\eqref{eq:Diff1inTriangle inequality3},~\eqref{eq:Diff2inTriangle inequality3} and ~\eqref{eq:noreset inequality}, we have
\begin{align}\label{eq:Diff1inTriangle UB}
&|\mu_{i_{l},W_{(l,c)}}-\mu_{i_{l},W_{:(l,c)}}| \nonumber \\
&\le |\mu_{i_{l},W_{(l,c)}}-\mu_{i_{l},W_{(l,c-1)}}|+|\mu_{i_{l},W_{(l,c-1)}}-\mu_{i_{l},W_{:(l,c)}}| \nonumber \\
&\le 2bKN+b+\frac{2^{l-1}-3+c}{2^{l-1}-2+c}\cdot 2\sqrt{\frac{\log(T^3)}{2N}} \nonumber \\
&< 2bKN+b+2\sqrt{\frac{\log(T^3)}{2N}}.
\end{align}

According to~\eqref{eq:Triangle inequality3}, it remains to derive the upper bound of $|\mu_{i_{l},W_{(l,c)}}-\mu_{i_{l},\widetilde{W}_{(l,c)}}|$. By Triangle Inequality,
\begin{align}\label{eq:eq:Diff1inTriangle UBprime}
&|\mu_{i_{l},W_{(l,c)}}-\mu_{i_{l},\widetilde{W}_{(l,c)}}|\le|\mu_{i_{l},W_{(l,c)}}-\mu_{i_{l},T_{m}}| \nonumber \\
&+|\mu_{i_{l},T_{m}}-\mu_{i_{l},\widetilde{W}_{(l,c)}}| \nonumber \\
&\le bKN+bKN=2bKN.
\end{align}
The above formula is based on the fact that $T_{m}$ belongs to the tuple $(l,c)$, thus the differences $|\mu_{i_{l},W_{(l,c)}}-\mu_{i_{l},T_{m}}|$ and $|\mu_{i_{l},T_{m}}-\mu_{i_{l},\widetilde{W}_{(l,c)}}|$ are at most $bKN$.

From \eqref{eq:Triangle inequality3}, \eqref{eq:Diff1inTriangle UB}, and \eqref{eq:eq:Diff1inTriangle UBprime}, we have:
\begin{align}\label{eq:conclusion_1}
|\mu_{i_{l},T_{m}}-\mu_{i_{l},W_{1}}|< 5bKN+2\sqrt{\frac{\log(T^3)}{2N}}+b.
\end{align}

Next, we derive the upper bound of $|\mu_{i_{l},T_{m}+1}-\mu_{i_{l},W_{2}}|$. Since $N\geq 16U_{m}$, we know 
\begin{align}\label{eq:conclusion_2}
|\mu_{i_{l},T_{m}+1}-\mu_{i_{l},W_{2}}|\le b\cdot 16KU_{m}\le bKN.
\end{align}

\subsection{Triggering the Detection and Concluding the Contradiction}

By~\eqref{eq:Hoeffding's inequality1} and ~\eqref{eq:conclusion_1}, we have
\begin{align}\label{eq:DiffinTriangle inequality4}
&|\mu_{i_{l},T_{m}}-\hat{\mu}_{i_{l},W_{1}}|\nonumber\\
\le& |\mu_{i_{l},T_{m}}-\mu_{i_{l},W_{1}}|+|\mu_{i_{l},W_{1}}-\hat{\mu}_{i_{l},W_{1}}|\nonumber \\
<& 5bKN+2\sqrt{\frac{\log(T^3)}{2N}}+b+\sqrt{\frac{\log(T^3)}{2|W_{i_{l},1}|}}.
\end{align}
Similarly, by~\eqref{eq:Hoeffding's inequality2} and ~\eqref{eq:conclusion_2}, we obtain
\begin{align}\label{eq:Diff1inTriangle inequality4}
&|\mu_{i_{l},T_{m}+1}-\hat{\mu}_{i_{l},W_{2}}|\nonumber\\
\le& |\mu_{i_{l},T_{m}+1}-\mu_{i_{l},W_{2}}|+|\mu_{i_{l},W_{2}}-\hat{\mu}_{i_{l},W_{2}}|\nonumber \\
\le& bKN+\sqrt{\frac{\log(T^3)}{2|W_{i_{l},2}|}}.
\end{align}
According to Definition~\ref{defn:Detectability}, we have
\begin{align}\label{eq:Diff2inTriangle inequality4}
&|\mu_{i_{l},T_{m}}-\mu_{i_{l},T_{m}+1}|\ge \epsilon_{m} \nonumber \\
&\ge \sqrt{\frac{\log (T^{3})}{2U_{m}}}+6bKN+b+2\sqrt{\frac{\log(T^3)}{2N}}
\end{align}

By  Triangle Inequality, we have 
\begin{align}\label{eq:Triangle inequality4}
&|\hat{\mu}_{i_{l}, W_1} - \hat{\mu}_{i_{l}, W_2}| \nonumber \\
&\ge |\mu_{i_{l},T_m} - \mu_{i_{l},T_m+1}| - |\mu_{i_{l},T_m} - \hat{\mu}_{i_{l}, W_1}|  \nonumber \\
&- |\mu_{i_{l},T_m+1} - \hat{\mu}_{i_{l}, W_2}|.
\end{align}
Substituting~\eqref{eq:DiffinTriangle inequality4}, ~\eqref{eq:Diff1inTriangle inequality4} and ~\eqref{eq:Diff2inTriangle inequality4} into ~\eqref{eq:Triangle inequality4}, note that $|W_{i_{l},1}|,|W_{i_{l},2}| \geq 16U_{m}$, we obtain:
\begin{align}
&|\hat{\mu}_{i_{l}, W_1} - \hat{\mu}_{i_{l}, W_2}| \nonumber \\
&> \sqrt{\frac{\log(T^3)}{2U_{m}}} -\sqrt{\frac{\log(T^3)}{2|W_{i_{l},1}|}} - \sqrt{\frac{\log(T^3)}{2|W_{i_{l},2}|}} \nonumber \\
&\ge \sqrt{\frac{\log(T^3)}{2|W_{i_{l},1}|}}  + \sqrt{\frac{\log(T^3)}{2|W_{i_{l},2}|}}=\varepsilon_{\text{cut}}^\delta.
\end{align}

In this case, since $|\hat{\mu}_{i_{l}, W_1} - \hat{\mu}_{i_{l}, W_2}|\ge \varepsilon_{\text{cut}}^\delta$, we know that our algorithm will reset at time $T_{m}+16KU_{m}$. Therefore, it contradicts the assumption that there is no detection between time step $X_{T_m}$ and $T_{m}+16KU_{m}$. So we conclude that 
\begin{align*}
\Prb(\cV) \ge 1-\frac{2K}{T}.
\end{align*}
\end{proof}

\section{Proof of Theorem~\ref{th:abrupt}}
\begin{definition}[Globally abrupt changes, Definition~19 in \cite{JMLR}]\label{Defn:GlobalChange}
Suppose that Assumption~\ref{assu:MChangePoints} holds, $c_a > 0$, and $\cT_c$ is defined in \eqref{eq:ChangePoints}. We define $\cT_{c}$ as a global change with constant $c_a$ if 
\begin{align}\label{eq:GlobalChange}
\max_{\substack{m\in[M]\\i,j \in \cK_m}} \frac{|\mu_{j,T_m+1} - \mu_{j,T_m}|}{|\mu_{i,T_m+1} - \mu_{i,T_m}|} \le c_a.
\end{align}
\end{definition}

\noindent\textit{Proof.} 
\subsection{Regret decomposition based on events}

Let $\cV$ be defined in Lemma~\ref{lem:Detection}, and we slightly abuse the notation and let $\setIn{\cV}$ denote $\setIn{\{Y_{T_{m+1}}, X_{T_{m}}\}_m \in \cV}$. Then, the expected regret within abrupt reset intervals can be decomposed as,
\begin{align}\label{eq:two_bound1}
&\mathbb{E}[\text{Reg}(T_{\text{abrupt}})]=\mathbb{E}\Big[\sum_{m=1}^{M} \sum_{t=X_{T_{m}}+1}^{Y_{T_{m+1}}} \text{Reg}(t)\Big] \nonumber\\
&= \mathbb{E}\Big[\sum_{m=1}^{M} \setIn{\cV^c}
 \sum_{t=X_{T_{m}}+1}^{Y_{T_{m+1}}} \text{Reg}(t)\Big] \nonumber \\
&+\mathbb{E}\Big[\sum_{m=1}^{M} \setIn{\cV} \sum_{t=X_{T_{m}}+1}^{Y_{T_{m+1}}} \text{Reg}(t)\Big] \nonumber\\
&\triangleq R_1 + R_2.
\end{align}

\subsection{Bounding $R_{1}$}
According to Lemma~\ref{lem:Detection}, we have $\Pr(\cV^c) = \mathcal{O}\!\left(\frac{K}{T}\right)$.  
Note that $\sum_{m=1}^{M} (Y_{T_{m+1}} - X_{T_{m}}) < T$.  
Furthermore, by Remark~\ref{remark:Hoeffding's inequality}, the expected reward satisfies $\mu_{i,t} \le \sqrt{\frac{2+p}{2}},\forall i \in[K],\forall t\in [0,T]$, which indicates that the variation in the expected reward between two consecutive time slots is at most $\sqrt{\frac{2+p}{2}}$.  
Therefore, $R_1$ can be upper bounded as follows:
\begin{align}\label{eq:R1_upper bound}
R_1=&\E\Big[\sum_{m=1}^{M} \setIn{\cV^c} \sum_{t=X_{T_{m}}+1}^{Y_{T_{m+1}}} \text{Reg}(t)\Big] \nonumber \\
<&  T \cdot \Pr(\cV^c) \cdot \sqrt{\frac{2+p}{2}}=2K \cdot \sqrt{\frac{2+p}{2}}=\mathcal{O}(1).
\end{align}

\subsection{Decomposing $R_2$}
Therefore, it remains to derive the upper bound of the second term $R_2$. By the linearity of expectation, $R_2$ can be further decomposed into two components as follows:
\begin{align}\label{eq:two_bound2}
R_{2}&=\E\Big[\sum_{m=1}^{M} \setIn{\cV} \sum_{t=X_{T_{m}}+1}^{Y_{T_{m+1}}} \text{Reg}(t)\Big] \nonumber \\
&= \sum_{m=1}^{M}\Big(\mathbb{E}\Big[ \setIn{\cV}   \sum_{t=X_{T_{m}}+1}^{T_{m}} \text{Reg}(t) \Big] \nonumber\\
&+\mathbb{E}\Big[\setIn{\cV} \sum_{t=T_{m}}^{Y_{T_{m+1}}} \text{Reg}(t)\Big]\Big)\nonumber \\
&\triangleq \sum_{m=1}^{M}(B_1 + B_2).
\end{align}

Although $B_1$ and $B_2$ depend on $m$, we suppress this dependence in the notation for simplicity, writing $B_1$ and $B_2$ when the context is clear. We next analyze these two components separately. In particular, before deriving the upper bound of $B_1$ and $B_2$, we first present a lemma that will facilitate the subsequent analysis.
\subsection{Relating regret to reward gap}
\begin{lemma}\label{lem:regret_gap_relation1}
Let $\mu_{i,X_{T_{m}}+1}$ denote the expected reward of arm $i$ at time $X_{T_{m}}+1$, and define the corresponding gap as 
\begin{align}\label{eq:Delta_i,m_1}
\Delta_{i,m}^{(1)} = \max_{j}\mu_{j,X_{T_{m}}+1} - \mu_{i,X_{T_{m}}+1}.
\end{align}
Furthermore, for any $X_{T_{m}}+1\le t\leq T_{m}$, define 
\begin{align}\label{eq:epsilon_{m}^{1}(t)}
\epsilon_{m}^{(1)}(t)=\max_{X_{T_{m}}+1\le s\le t} \max_{i\in[K]} |\mu_{i,s}-\mu_{i,X_{T_{m}}+1}|,
\end{align}
which quantifies the maximum drift in the expected rewards within the interval $[X_{T_{m}}+1,T_{m}]$. 
For ease of exposition, we let $i=I(t)$ denote the arm selected at time slot $t$.  
This substitution does not affect generality, since the instantaneous regret $\text{Reg}(t)=\max_{j}\mu_{j,t}-\mu_{I(t),t}$ depends solely on the selected arm at time $t$.  
The relationship between the instantaneous regret $\text{Reg}(t)$ and the gap $\Delta_{i,m}^{(1)}$ satisfies
\begin{align}\label{eq:Diff_Reg and Delta1}
|\text{Reg}(t)-\Delta_{i,m}^{(1)}| \le  2\epsilon_{m}^{(1)}(t).
\end{align}
\end{lemma}
\noindent\textit{Proof.} Next, we will prove Lemma~\ref{lem:regret_gap_relation1}.

\begin{figure*}[t]
\begin{align}
    &\text{Reg}(t)=\max_{j}\mu_{j,t}-\mu_{i,t}=\max_{j}\mu_{j,X_{T_{m}}+1}-\mu_{i,X_{T_{m}}+1}+\mu_{i,X_{T_{m}}+1}-\mu_{i,t}+\max_{j}\mu_{j,t}-\max_{j}\mu_{j,X_{T_{m}}+1}\nonumber \\
    &=\Delta_{i,m}^{(1)}+\mu_{i,X_{T_{m}}+1}-\mu_{i,t}+\max_{j}\mu_{j,t}-\max_{j}\mu_{j,X_{T_{m}}+1}.\label{eq:reg_delta}
\end{align}
\vspace{0.5em}
\noindent\rule{\textwidth}{0.4pt}
\end{figure*}

According to the definitions of $\text{Reg}(t)$ and $\Delta_{i,m}^{(1)}$ given in~\eqref{eq:Reg} and~\eqref{eq:Delta_i,m_1}, 
their relationship can be expressed as in~\eqref{eq:reg_delta}. By \eqref{eq:reg_delta}, we obtain
\begin{align}\label{eq:reg_minus_delta}
\text{Reg}(t)-\Delta_{i,m}^{(1)}&=\mu_{i,X_{T_{m}}+1}-\mu_{i,t}\nonumber\\
&+\max_{j}\mu_{j,t}-\max_{j}\mu_{j,X_{T_{m}}+1}.
\end{align}
We next derive an upper bound for the right-hand side of \eqref{eq:reg_minus_delta}. According to the definition of $\epsilon_{m}^{(1)}(t)$ in~\eqref{eq:epsilon_{m}^{1}(t)}, it follows that
\begin{align}\label{eq:Diff1_triangle Inequality}
|\mu_{i,X_{T_{m}}+1}-\mu_{i,t}|\le \epsilon_{m}^{(1)}(t).
\end{align}
Without loss of generality, assume that at time slot $t$, the arm $I$ achieves the largest expected reward, i.e., $\max_{j}\mu_{j,t}=\mu_{I,t}$. Similarly, let $J$ denote the arm with the largest expected reward at time $X_{T_{m}}+1$, such that $\max_{j}\mu_{j,X_{T_{m}}+1} = \mu_{J,X_{T_{m}}+1}$. Then we have
\begin{align}\label{eq:Diff2_triangle Inequality}
    |\max_{j}\mu_{j,t}-\max_{j}\mu_{j,X_{T_{m}}+1}|\le \epsilon_{m}^{(1)}(t).
\end{align}
The derivation of~\eqref{eq:Diff2_triangle Inequality} proceeds as follows.  
Based on the above assumptions, we have $\mu_{I,t}\ge \mu_{J,t}$ and $\mu_{J,X_{T_{m}}+1}\ge \mu_{I,X_{T_{m}}+1}$. 
We analyze the possible relationships among $\mu_{I,t}$, $\mu_{I,X_{T_m}+1}$, and $\mu_{J,X_{T_{m}}+1}$ to establish the desired inequality.

Consider the case where $\mu_{I, t} \leq \mu_{I, X_{T_m}+1} \le \mu_{J, X_{T_m}+1}$.  
To illustrate this relationship, we present the diagram in Fig.~\ref{fig:reward_relation1}. 
\begin{center}
\begin{tikzpicture}
    \draw[->] (-1,0) -- (5,0) node[right] {$x$};
    \draw (0,0.1) -- (0,-0.1) node[below] {$\mu_{J,t}$};
    \draw (1,0.1) -- (1,-0.1) node[below] {$\mu_{I,t}$};
    \draw (2.5,0.1) -- (2.5,-0.1) node[below] {$\mu_{I,X_{T_{m}}+1}$};
    \draw (4.5,0.1) -- (4.5,-0.1) node[below] {$\mu_{J,X_{T_{m}}+1}$};
\end{tikzpicture}
\captionof{figure}{relationship diagram if $\mu_{I, t} \le \mu_{I, X_{T_m}+1} \le \mu_{J, X_{T_m}+1}$.}
\label{fig:reward_relation1}
\end{center}
As illustrated in Fig~\ref{fig:reward_relation1}, we have
\begin{align*}
    &|\max_{j}\mu_{j,t}-\max_{j}\mu_{j,X_{T_{m}}+1}|=|\mu_{I,t}-\mu_{J,X_{T_{m}}+1}| \\
    &\le|\mu_{J,t}-\mu_{J,X_{T_{m}}+1}|\le \epsilon_{m}^{(1)}(t).
\end{align*}
Next, consider the case $\mu_{I,X_{T_{m}}+1}\le \mu_{I,t} \le \mu_{J,X_{T_{m}}+1}$.
The corresponding relationships are shown in Fig.~\ref{fig:reward_relation2} and Fig.~\ref{fig:reward_relation3}.
\begin{center}
\begin{tikzpicture}
    \draw[->] (-1,0) -- (5,0) node[right] {$x$}; 
    \draw (0,0.1) -- (0,-0.1) node[below] {$\mu_{J,t}$};
    \draw (1.5,0.1) -- (1.5,-0.1) node[below] {$\mu_{I,X_{T_{m}}+1}$};
    \draw (3,0.1) -- (3,-0.1) node[below] {$\mu_{I,t}$};
    \draw (4.5,0.1) -- (4.5,-0.1) node[below] {$\mu_{J,X_{T_{m}}+1}$};
\end{tikzpicture}
\captionof{figure}{relationship diagram if $\mu_{I,X_{T_{m}}+1}\le \mu_{I,t} \le \mu_{J,X_{T_{m}}+1}$.}
\label{fig:reward_relation2}
\end{center}
\begin{center}
\begin{tikzpicture}
    \draw[->] (-1,0) -- (5,0) node[right] {$x$}; 
    \draw (0,0.1) -- (0,-0.1) node[below] {$\mu_{I,X_{T_{m}}+1}$};
    \draw (1.5,0.1) -- (1.5,-0.1) node[below] {$\mu_{J,t}$};
    \draw (3,0.1) -- (3,-0.1) node[below] {$\mu_{I,t}$};
    \draw (4.5,0.1) -- (4.5,-0.1) node[below] {$\mu_{J,X_{T_{m}}+1}$};
\end{tikzpicture}
\captionof{figure}{relationship diagram if $\mu_{I,X_{T_{m}}+1}\le \mu_{I,t} \le \mu_{J,X_{T_{m}}+1}$.}
\label{fig:reward_relation3}
\end{center}
As illustrated in Fig.~\ref{fig:reward_relation2} and~\ref{fig:reward_relation3}, we similarly have
\begin{align*}
    &|\max_{j}\mu_{j,t}-\max_{j}\mu_{j,X_{T_{m}}+1}|=|\mu_{I,t}-\mu_{J,X_{T_{m}}+1}|\\
    &\le|\mu_{J,t}-\mu_{J,X_{T_{m}}+1}|\le \epsilon_{m}^{(1)}(t).
\end{align*}
Finally, we consider $\mu_{I,t} \ge \mu_{J,X_{T_{m}}+1}\ge \mu_{I,X_{T_{m}}+1}$.
Note that multiple possible relationship may exist among $\mu_{J,t}$, $\mu_{J,X_{T_{m}}+1}$, and $\mu_{I,X_{T_{m}}+1}$; however, the absence of $\mu_{J,t}$ in the comparison does not affect the subsequent analysis. We focus only on $\mu_{J,X_{T_{m}}+1}$, $\mu_{I,X_{T_{m}}+1}$, and $\mu_{I,t}$, as illustrated in Fig.~\ref{fig:reward_relation4}.
\begin{center}
\begin{tikzpicture}
    \draw[->] (-1,0) -- (5,0) node[right] {$x$}; 
    \draw (0,0.1) -- (0,-0.1) node[below] {$\mu_{I,X_{T_{m}}+1}$};
    \draw (2,0.1) -- (2,-0.1) node[below] {$\mu_{J,X_{T_{m}}+1}$};
    \draw (4,0.1) -- (4,-0.1) node[below] {$\mu_{I,t}$};
\end{tikzpicture}
\captionof{figure}{relationship diagram if $\mu_{I,t} \ge \mu_{J,X_{T_{m}}+1}\ge \mu_{I,X_{T_{m}}+1}$.}
\label{fig:reward_relation4}
\end{center}
As illustrated in Fig~\ref{fig:reward_relation4}, we get
\begin{align*}
    &|\max_{j}\mu_{j,t}-\max_{j}\mu_{j,X_{T_{m}}+1}|=|\mu_{I,t}-\mu_{J,X_{T_{m}}+1}|\\
    &\le|\mu_{I,t}-\mu_{I,X_{T_{m}}+1}|\le \epsilon_{m}^{(1)}(t).
\end{align*}
Combining~\eqref{eq:Diff1_triangle Inequality} and~\eqref{eq:Diff2_triangle Inequality} and applying the triangle inequality, we finally obtain
\begin{align*}
|\text{Reg}(t)-\Delta_{i,m}^{(1)}|\le& |\mu_{i,X_{T_{m}}+1}-\mu_{i,t}| \nonumber \\
+&|\max_{j}\mu_{j,t}-\max_{j}\mu_{j,X_{T_{m}}+1}|\nonumber\\
\le&\epsilon_{m}^{(1)}(t)+\epsilon_{m}^{(1)}(t) =2 \epsilon_{m}^{(1)}(t).
\end{align*}
\par\noindent\hfill$\square$\par
\subsection{Bounding $\sum_{m=1}^{M}{B_1}$: regret before change points under event $\cV$}
After completing the proof of Lemma~\ref{lem:regret_gap_relation1}, we will continue to prove Theorem~\ref{th:abrupt}.

Based on lemma~\ref{lem:regret_gap_relation1}, we derive an upper bound of $B_{1}$. 
Let $N_{i,m}^{(1)}=\sum_{t=X_{T_m}+1}^{T_m} \setIn{I(t)=i}$ denote the number of times that arm $i$ is selected within the interval $[X_{T_m} + 1, T_m]$, 
and let $\mathcal{H}_m^{(1)}$ be the natural filtration (history information) until the $m$-th abrupt reset. According to the relationship established in~\eqref{eq:Diff_Reg and Delta1}, the instantaneous regret satisfies 
$\text{Reg}(t)\le\Delta_{i,m}^{(1)}+2 \epsilon_{m}^{(1)}(t)$. 
Therefore, $B_1$ can be upper bounded as follows:
\begin{align}\label{eq:upper bound1_X_{T_{m}}_T_m}
    &B_{1}=\mathbb{E}\Big[\mathbb{E}\Big[\setIn{\cV}\sum_{t=X_{T_{m}}+1}^{T_{m}} \text{Reg}(t)\mid \mathcal{H}_{m}^{(1)}\Big]\Big] \nonumber \\
    &\le \mathbb{E}\Big[\sum_{i:\Delta_{i,m}^{(1)}>0}\Big(\Delta_{i,m}^{(1)}\mathbb{E}\left[N_{i,m}^{(1)}\mid \mathcal{H}_{m}^{(1)}\right] \nonumber \\
    &+\mathbb{E}\left[N_{i,m}^{(1)}\max_{t} 2\epsilon_{m}^{(1)}(t)\mid \mathcal{H}_{m}^{(1)}\right]\Big)\Big].
\end{align}
Meanwhile, by invoking the definition of Drift-Tolerant Regret in Definition~\ref{defn:Drift-tolerant regret} and Remark~\ref{remark:upperbound of Reg(s)}, we further obtain
\begin{align}\label{eq:upper bound2_X_{T_{m}}_T_m}
    &B_{1}=\mathbb{E}\Big[\mathbb{E}\Big[\setIn{\cV}\sum_{t=X_{T_{m}}+1}^{T_{m}} \text{Reg}(t)\mid \mathcal{H}_{m}^{(1)}\Big]\Big] \nonumber \\
    &\le \mathbb{E}\Big[\sum_{i:\Delta_{i,m}^{(1)}>0}\Big(\mathcal{O}\Big(\frac{\log T}{\Delta_{i,m}^{(1)}}\Big)\nonumber \\
    &+\mathbb{E}\Big[N_{i,m}^{(1)} \max_{t}2\epsilon_{m}^{(1)}(t)\mid \mathcal{H}_{m}^{(1)}\Big]\Big)\Big].
\end{align}
Combining~\eqref{eq:upper bound1_X_{T_{m}}_T_m} and~\eqref{eq:upper bound2_X_{T_{m}}_T_m}, we finally obtain:
\begin{align}\label{eq:two_bound3}
    &B_{1}\le \nonumber \\
    &\mathbb{E}\Big[\sum_{i:\Delta_{i,m}^{(1)}>0}
        \min\Big\{\Delta_{i,m}^{(1)}\mathbb{E}\left[N_{i,m}^{(1)}\mid \mathcal{H}_{m}^{(1)}\right],\mathcal{O}\Big(\frac{\log T}{\Delta_{i,m}^{(1)}}\Big)
        \Big\} \Big] \nonumber \\
    &+ \mathbb{E}\Big[\sum_{i:\Delta_{i,m}^{(1)}>0}\mathbb{E}\Big[N_{i,m}^{(1)}\max_{t}2\epsilon_{m}^{(1)}(t)\mid \mathcal{H}_{m}^{(1)}\Big]\Big]
     \nonumber \\
    &\triangleq C_1 + C_2.
\end{align}

Similar to $B_{1}$, we write $C_{1}$ and $C_{2}$ for simplicity, suppressing their explicit dependence on $m$ when the context is clear.
\subsubsection{Bounding $C_1$}
Then, we provide upper bounds for $C_1$ and $C_2$. We begin with the term $C_1$. By applying the inequality $\min(a,b)\le \sqrt{ab}$, $C_1$ can be upper bounded as
\begin{align*}
    C_1\le \mathbb{E}\Big[\sum_{i:\Delta_{i,m}^{(1)}>0}\mathcal{O}\Big(\sqrt{\mathbb{E}\Big[N_{i,m}^{(1)}\mid \mathcal{H}_{m}^{(1)}\Big] \log T}\Big)\Big].
\end{align*}
Next, by Jensen’s inequality for the concave function $\sqrt{x}$ (i.e., $\mathbb{E}[\sqrt{x}]\le \sqrt{\mathbb{E}[x]}$) and the law of total expectation, it follows 
\begin{align*}
&\mathbb{E}\Big[\sum_{i:\Delta_{i,m}^{(1)}>0}\mathcal{O}\Big(\sqrt{\mathbb{E}\Big[N_{i,m}^{(1)}\mid \mathcal{H}_{m}^{(1)}\Big]\log T}\Big)\Big] \nonumber \\
&\le\sum_{i:\Delta_{i,m}^{(1)}>0}\mathcal{O}\Big(\sqrt{\mathbb{E}\Big[\mathbb{E}\Big[N_{i,m}^{(1)}\mid \mathcal{H}_{m}^{(1)}\Big]\Big] \log T}\Big) \nonumber \\
&=\sum_{i:\Delta_{i,m}^{(1)}>0}\mathcal{O}\Big(\sqrt{\mathbb{E}\Big[N_{i,m}^{(1)}\Big] \log T}\Big).
\end{align*}
Since the summation over arms with $\Delta_{i,m}^{(1)}>0$ is a subset of all arms, we have
\begin{align*}
    \sum_{i:\Delta_{i,m}^{(1)}>0}\mathcal{O}\Big(\sqrt{\mathbb{E}\big[N_{i,m}^{(1)}\big]  \log T}\Big)<\sum_{i}\mathcal{O}\Big(\sqrt{\mathbb{E}\Big[N_{i,m}^{(1)}\Big] \log T}\Big).
\end{align*}
Finally, by applying the Cauchy--Schwarz inequality and noting that $\sum_{i}\mathbb{E}\left[N_{i,m}^{(1)}\right] = T_{m}-X_{T_{m}}$, we have
\begin{align*}
    \sum_{i}\mathcal{O}\Big(\sqrt{\mathbb{E}\Big[N_{i,m}^{(1)}\Big] \log T}\Big)\le \mathcal{O}\Big(\sqrt{K(T_{m}-X_{T_{m}})\log T}\Big).
\end{align*}
Therefore, the upper bound of $C_1$ can be expressed as
\begin{align}\label{eq:upper_bound inequality1}
    &C_1=\nonumber \\
    &\mathbb{E}\Big[\sum_{i:\Delta_{i,m}^{(1)}>0}\min\Big\{\Delta_{i,m}^{(1)}\mathbb{E}\left[N_{i,m}^{(1)}\mid \mathcal{H}_{m}^{(1)}\right],\mathcal{O}\Big(\tfrac{\log T}{\Delta_{i,m}^{(1)}}\Big)\Big\}\Big] \nonumber \\
    &<\mathcal{O}\Big(\sqrt{K(T_{m}-X_{T_{m}})\log T}\Big).
\end{align}

\subsubsection{Bounding $C_2$}
Next, we derive an upper bound for $C_2$. 
Lemma~\ref{lem:diffenrence} indicates that, within each gradual segment, the difference in the expected rewards between any two time slots is upper bounded. 
According to the definition of $\epsilon_m^{(1)}(t)$ in~\eqref{eq:epsilon_{m}^{1}(t)}, there exists a constant
\begin{align}\label{eq:DefG}
G = \frac{\log T}{c_{g}}\Big(2bKN+8\sqrt{\tfrac{\log(T^{3})}{2N}}\Big)+b\log T.
\end{align}
such that $\max_{t}\epsilon_m^{(1)}(t)\le G$. Using this bound and linearity of expectation, we obtain
\begin{align*}
C_2&=\mathbb{E}\Big[\sum_{i:\Delta_{i,m}^{(1)}>0}\mathbb{E}\left[N_{i,m}^{(1)}\max_{t}2\epsilon_m^{1}(t)\mid \mathcal{H}_{m}^{(1)}\right]\Big] \nonumber \\
    &\le 2G \cdot\mathbb{E}\Big[\sum_{i:\Delta_{i,m}^{(1)}>0}\mathbb{E}\left[N_{i,m}^{(1)}\mid\mathcal{H}_{m}^{(1)}\right] \Big] \\
    &=2G\sum_{i:\Delta_{i,m}^{(1)}>0}\mathbb{E}\left[N_{i,m}^{(1)}\right].
\end{align*}
Since $\{i:\Delta_{i,m}^{(1)}>0\}\subseteq [K]$ and $\sum_{i}\mathbb{E}[N_{i,m}^{(1)}]=T_m-X_{T_m}$, it follows that
\begin{align}\label{eq:2G}
2G \sum_{i:\Delta_{i,m}^{(1)}>0}\E\left[N_{i,m}^{(1)}\right]<&2G \sum_{i}\mathbb{E}\left[N_{i,m}^{(1)}\right]\nonumber\\
=&2G (T_{m}-X_{T_{m}}).
\end{align}

\subsubsection{Summation over all change points}
according to \eqref{eq:two_bound3}, we obtain
\begin{align*}
    &\sum_{m=1}^{M} B_1=\sum_{m=1}^{M}\mathbb{E}\Big[\setIn{\cV} \cdot \sum_{t = X_{T_m} + 1}^{T_m} \text{Reg}(t)\Big] \nonumber \\
    &\le\sum_{m=1}^{M} C_1+\sum_{m=1}^{M}C_2.
\end{align*}
Given that $\sum_{m=1}^{M} (T_{m}-X_{T_{m}})<T$ and the upper bound of $C_{1}$ is provided in~\eqref{eq:upper_bound inequality1}, applying the Cauchy-Schwarz inequality yields the following bound
\begin{align}
\sum_{m=1}^{M} C_1&< \sum_{m=1}^{M}\mathcal{O}\Big(\sqrt{K(T_{m}-X_{T_{m}})\log T}\Big) \nonumber \\
&<\mathcal{O}\Big(\sqrt{KMT\log T}\Big)=\mathcal{O}\Big(\sqrt{T\log T}\Big) \label{eq:gradual_regret_upper1};
\end{align}
and incorporating \eqref{eq:2G}, we have
\begin{align}\label{eq:gradual_regret_upper2}
\sum_{m=1}^{M}C_2<2G\sum_{m=1}^{M}(T_m - X_{T_m})<2GT.
\end{align}
By the condition $N=\mathcal{O}\left((bK)^{-\frac{2}{3}}\right)$ stated in Theorem~\ref{th:abrupt}, the expression of $G$ can be correspondingly simplified as follows:
\begin{align*}
&G=2\Big(\frac{\log T}{c_{g}}\Big(2bKN+8\sqrt{\frac{\log (T^{3})}{2N}}\Big)+b\log T\Big) \nonumber \\
&=2\Big(\frac{\log T}{c_{g}}\mathcal{O}\Big((bK)^{\frac{1}{3}}\Big)\Big(2+8\sqrt{\frac{\log (T^{3})}{2}}\Big)+b\log T\Big) \nonumber \\
&=\mathcal{O}\Big((bK)^{\frac{1}{3}}\cdot (\log T)^{\frac{3}{2}}+b\log T\Big).
\end{align*}
Given that $b = T^{-d}$ and $K$ is a constant,~\eqref{eq:gradual_regret_upper2} can be rewritten as
\begin{align}\label{eq:gradual_regret_upper3}
\sum_{m=1}^{M}C_2<\mathcal{O}\Big(T^{1-\frac{d}{3}} (\log T)^{\frac{3}{2}}\Big).
\end{align}
Therefore, combining~\eqref{eq:gradual_regret_upper1} and~\eqref{eq:gradual_regret_upper3}, $\sum_{m=1}^{M} B_1$ is upper bounded by
\begin{align}\label{eq:sum_B1 upper bound}
\sum_{m=1}^{M} B_1<&\mathcal{O}\Big(\sqrt{T\log T}\Big)+\mathcal{O}\Big(T^{1-\frac{d}{3}} (\log T)^{\frac{3}{2}}\Big).
\end{align}

\subsection{Bounding $\sum_{m=1}^{M}{B_2}$: regret after change points under event $\cV$}

We next derive an upper bound for $\sum_{m=1}^{M} B_2$, 
by following the same analytical framework used in establishing the upper bound of $\sum_{m=1}^{M} B_1$.

Let $\mu_{i,T_{m}}$ denote the expected reward of arm $i$ at time $T_{m}$, and define the corresponding gap as 
\begin{align}\label{eq:Delta_i,m_2}
    \Delta_{i,m}^{(2)}=\max_{j} \mu_{j,T_{m}} - \mu_{i,T_{m}}.
\end{align}
Furthermore, for $T_m\le t\le Y_{T_{m+1}}$, define 
\begin{align}\label{eq:epsilon_{m}^{2}(t)}
    \epsilon_{m}^{(2)}(t) = \max_{T_{m}\leq s \leq t} \max_{i\in [K]} \left| \mu_{i,s} - \mu_{i,T_{m}} \right|,
\end{align}
which represents the maximum amount of drift within the interval $[T_{m}, Y_{T_{m+1}}]$. Similar for the proof of Lemma~\ref{lem:regret_gap_relation1}, the instantaneous regret $\text{Reg}(t)$ and the gap $\Delta_{i,m}^{(2)}$ satisfy the following relationship:
\begin{align}\label{eq:Diff_Reg and Delta2}
    |\text{Reg}(t)-\Delta_{i,m}^{(2)}| \le  2\epsilon_{m}^{(2)}(t).
\end{align}
 Let $N_{i,m}^{(2)} = \sum_{t=T_{m}}^{Y_{T_{m+1}}} \setIn{I(t)=i}$ denote the number of times arm $i$ is pulled between $T_{m}$ and $Y_{T_{m+1}}$, and let $\mathcal{H}_{m}^{(2)}$ the natural filtration (history information) until $T_{m}$. Following a similar method as in~\eqref{eq:two_bound3}, we can decompose the expected regret $B_2$ as
\begin{align}\label{eq:two_bound4}
    &B_{2}= \mathbb{E}\Big[\mathbb{E}\Big[\setIn{ \cV}\sum_{t=T_{m}}^{Y_{T_{m+1}}} \text{Reg}(t)\mid \mathcal{H}_{m}^{(2)}\Big]\Big] \nonumber \\
    &\le \mathbb{E}\Big[\sum_{i:\Delta_{i,m}^{(2)}>0}\min\Big\{\Delta_{i,m}^{(2)}\mathbb{E}\left[N_{i,m}^{(2)}\mid \mathcal{H}_{m}^{(2)}\right], \mathcal{O}\Big(\frac{\log T}{\Delta_{i,m}^{(2)}}\Big)\Big\}\Big]  \nonumber \\
    &+ \mathbb{E}\Big[\sum_{i:\Delta_{i,m}^{(2)}>0}\mathbb{E}\Big[N_{i,m}^{(2)}\max_{t}2\epsilon_{m}^{(2)}(t)\mid \mathcal{H}_{m}^{(2)}\Big]\Big] \nonumber \\
    &\triangleq D_1 + D_2.
\end{align}

Similar to $C_{1}$ and $C_{2}$, although $D_{1}$ and $D_{2}$ depend on $m$, we suppress this dependence in the notation for simplicity, writing $D_{1}$ and $D_{2}$ when the context is clear.

\subsubsection{Bounding $D_1$}
For the upper bound of $D_1$, similar for the proof of~\eqref{eq:upper_bound inequality1}, we obtain
\begin{align*}
    D_1< \mathcal{O}\left(\sqrt{K(Y_{T_{m+1}}-T_{m})\log T}\right).
\end{align*}
Summing over all $m$, and using the Cauchy-Schwarz inequality together with the fact that $\sum_{m=1}^{M} (Y_{T_{m+1}} - T_{m}) < T$, we have
\begin{align}\label{eq:gradual_regret_upper4}
    &\sum_{m=1}^{M} D_1\le \sum_{m=1}^{M} \mathcal{O}\left(\sqrt{K(Y_{T_{m+1}}-T_{m})\log T}\right) \nonumber \\
    &<\mathcal{O}\left(\sqrt{KMT\log T}\right)=\mathcal{O}\left(\sqrt{T\log T}\right).
\end{align}

\subsubsection{Bounding $D_2$}
Then, for the upper bound of $D_2$, we first derive an upper bound for $\epsilon_{m}^{(2)}(t)$. According to Definitions~\ref{Defn:GlobalChange} and Assumption~\ref{assu:epsilon_upper_bound}, it holds that for $\forall i\in[K]$,
\begin{align}\label{eq:Diff_Triangle inequality1}
    &|\mu_{i,T_{m}}-\mu_{i,T_{m}+1}| \nonumber \\
    &\le c_{a} c_{u}\Big(\sqrt{\frac{\log (T^{3})}{2U_{m}}}+6bKN+2\sqrt{\frac{\log(T^3)}{2N}}+b\Big).
\end{align}
Moreover, within each gradual segment, the expected reward evolves at most at rate $b$. Therefore, 
\begin{align}\label{eq:Diff_Triangle inequality2}
|\mu_{i,Y_{T_{m+1}}}-\mu_{i,T_{m}+1}|\le b\cdot16KU_{m}\le bKN,\,\,\forall i\in[K].
\end{align}
Combining~\eqref{eq:Diff_Triangle inequality1} and~\eqref{eq:Diff_Triangle inequality2} and applying the triangle inequality on \eqref{eq:epsilon_{m}^{2}(t)} yields,
\begin{align}\label{eq:epsilon_{m}^{2}(t)_upper bound}
&\epsilon_{m}^{(2)}(t) \nonumber \\
&\le\max_{i}|\mu_{i,T_{m}}-\mu_{i,T_{m}+1}|+\max_{i}|\mu_{i,Y_{T_{m+1}}}-\mu_{i,T_{m}+1}| \nonumber \\
&=c_{a}c_{u}\Big(\sqrt{\frac{\log (T^{3})}{2U_{m}}}+6bKN+2\sqrt{\frac{\log(T^3)}{2N}}+b\Big)+bKN\nonumber\\
&\triangleq D.
\end{align}
Substituting the bound in~\eqref{eq:epsilon_{m}^{2}(t)_upper bound} into the expression of $D_2$, we have
\begin{align*}
    D_2\le 2D \cdot \mathbb{E}\Big[\sum_{i:\Delta_{i,m}^{(2)}>0}\mathbb{E}\Big[N_{i,m}^{(2)}\mid \mathcal{H}_{m}^{(2)}\Big]\Big].
\end{align*}
Applying the law of total expectation and the linearity of expectation, it follows that
\begin{align*}
     &2D \cdot \mathbb{E}\Big[\sum_{i:\Delta_{i,m}^{(2)}>0}\mathbb{E}\left[N_{i,m}^{(2)}\mid \mathcal{H}_{m}^{(2)}\right]\Big] \nonumber \\
     &\le 2D \cdot \sum_{i:\Delta_{i,m}^{(2)}>0}\mathbb{E}\left[N_{i,m}^{(2)}\right].
\end{align*}
Since the summation over $\{i:\Delta_{i,m}^{(2)}>0\}$ is a subset of all arms, and 
$\sum_{i}\mathbb{E}[N_{i,m}^{(2)}] = Y_{T_{m+1}}-T_{m}$, we further obtain
\begin{align*}
    &2D \cdot \sum_{i:\Delta_{i,m}^{(2)}>0}\mathbb{E}\left[N_{i,m}^{(2)}\right]< 2D \cdot \sum_{i}\mathbb{E}\left[N_{i,m}^{(2)}\right] \nonumber \\
    &=2D\cdot(Y_{T_{m+1}}-T_{m}).
\end{align*}
Hence, $D_2$ is bounded by
\begin{align}\label{eq:two_bound5}
D_2 <2D\cdot(Y_{T_{m+1}}-T_{m}).
\end{align}

We next denote
\begin{align*}
E_1=&2c_{a}c_{u}\sqrt{\frac{\log (T^{3})}{2U_{m}}}\\
E_2=&2\Big[c_{a}c_{u}\Big(6bKN+2\sqrt{\frac{\log(T^3)}{2N}}+b\Big)+bKN\Big].
\end{align*}
Then, \eqref{eq:two_bound5} can be re-written as
\begin{align}\label{eq:two_bound6}
&D_2 < (E_1+E_2)(Y_{T_{m+1}}-T_{m}).
\end{align}

We derive an upper bound of $\sum_{m=1}^{M}(Y_{T_{m+1}}-T_{m}) E_{1}$. By utilizing the inequality $Y_{T_{m+1}}-T_{m} \le 16KU_{m}$, we have
\begin{align*}
&(Y_{T_{m+1}}-T_{m})\cdot E_{1}=2c_{a}c_{u}\sqrt{\frac{\log (T^{3})}{2U_{m}}} \cdot (Y_{T_{m+1}}-T_{m})  \nonumber \\
&\le 2c_{a}c_{u} \cdot \sqrt{\frac{\log (T^{3})}{2U_{m}}}\sqrt{16KU_m(Y_{T_{m+1}}-T_m)}\\
&=2c_{a}c_{u} \cdot \sqrt{8K(Y_{T_{m+1}}-T_{m})\log(T^{3})}.
\end{align*}
By applying the Cauchy–-Schwarz inequality and noting that $\sum_{m=1}^{M}(Y_{T_{m+1}} - T_{m}) < T$, where $K$ and $M$ are constants, we obtain
\begin{align}\label{eq:D2_part1_upper bound}
    &\sum_{m=1}^{M}(Y_{T_{m+1}}-T_{m})\cdot E_{1}  \nonumber \\
    &\le \sum_{m=1}^{M} 2c_{a}c_{u} \cdot \sqrt{8K(Y_{T_{m+1}}-T_{m}) \log (T^{3})} \nonumber \\
    &<2c_{a}c_{u} \cdot \sqrt{8KMT \log(T^{3})}=\mathcal{O}(\sqrt{T \log T}).
\end{align}

We derive an upper bound for $\sum_{m=1}^{M}(Y_{T_{m+1}}-T_{m})\cdot E_{2}$.  
Since $N=\mathcal{O}\left((bK)^{-2/3}\right)$ and $K$ is a constant, it follows that
\begin{align*}
&E_{2}=2c_{a}c_{u}\left(6bKN+2\sqrt{\frac{\log(T^3)}{2N}}+b\right)+bKN \nonumber \\
&=\mathcal{O}(b^{\frac{1}{3}}\cdot (\log T)^{\frac{1}{2}}+b).
\end{align*}
Substituting $b=T^{-d}$ yields
\begin{align}\label{eq:E_upper bound}
E_{2}=\mathcal{O}(T^{-\frac{d}{3}}(\log T)^{\frac{1}{2}}).
\end{align}
Then, using~\eqref{eq:E_upper bound} and the fact that $\sum_{m=1}^{M}(Y_{T_{m+1}}-T_{m})<T$, we obtain
\begin{align}\label{eq:D2_part2_upper bound}
    &\sum_{m=1}^{M} (Y_{T_{m+1}}-T_{m})\cdot E_{2}<T \cdot E_{2} \nonumber \\
    &=T \cdot \mathcal{O}(T^{-\frac{d}{3}}(\log T)^{\frac{1}{2}})=\mathcal{O}(T^{1-\frac{d}{3}}(\log T)^{\frac{1}{2}}).
\end{align}
Combining~\eqref{eq:two_bound6},~\eqref{eq:D2_part1_upper bound} and~\eqref{eq:D2_part2_upper bound}, we obtain
\begin{align}\label{eq:gradual_regret_upper5}
    &\sum_{m=1}^{M} D_2=\sum_{m=1}^{M}(Y_{T_{m+1}}-T_{m})\cdot (E_{1}+E_{2})\nonumber \\
    &<\mathcal{O}(\sqrt{T\log T})+\mathcal{O}(T^{1-\frac{d}{3}} (\log T)^{\frac{1}{2}}).
\end{align}

\subsubsection{Final summation of $\sum_{m=1}^{M}B_2$}
Finally, combining~\eqref{eq:two_bound4},~\eqref{eq:gradual_regret_upper4} and~\eqref{eq:gradual_regret_upper5}, we obtain:
\begin{align}\label{eq:sum_B2 upper bound}
    &\sum_{m=1}^{M} B_2=\sum_{m=1}^{M}(D_{1}+D_{2})\nonumber \\
    &<2\mathcal{O}(\sqrt{T\log T})+\mathcal{O}(T^{1-\frac{d}{3}} (\log T)^{\frac{1}{2}}) \nonumber \\
    &=\mathcal{O}(\sqrt{T\log T})+\mathcal{O}(T^{1-\frac{d}{3}} (\log T)^{\frac{1}{2}}).
\end{align}

\subsection{Final regret bound}

Therefore, according to~\eqref{eq:two_bound1},~\eqref{eq:R1_upper bound},~\eqref{eq:two_bound2},~\eqref{eq:sum_B1 upper bound}and~\eqref{eq:sum_B2 upper bound}, the expected regret incurred during the abrupt reset intervals can be bounded as
\begin{align*}
\E[\text{Reg}(T_{\text{abrupt}})]&=R_1+\sum_{m=1}^{M}(B_1+B_2) \nonumber \\
    &<\mathcal{O}(\sqrt{T\log T})+\mathcal{O}(T^{1-\frac{d}{3}} (\log T)^{\frac{3}{2}}).
\end{align*}

\par\noindent\hfill$\square$\par

\section{Proof of Lemma~\ref{lem:reset_number}}\label{Appe:reset_number}
Let 
\begin{align}\label{eq:F_1}
F_{1}=(\frac{\sqrt{3}-\sqrt{2+d}}{2\sqrt{2}}\sqrt{\log T})^{\frac{2}{3}}.
\end{align}

We define the following events:
\begin{align*}
\mathcal{Y}_{j}(t)
&= \bigcup_{\substack{W_{1},W_{2}:\\ W(t)=W_{1}\cup W_{2},\, j\in[K]}}
        \Big\{
            |W_{1}|\le F_{1}b^{-\frac{2}{3}},\;
            |W_{2}|\le F_{1}b^{-\frac{2}{3}},\nonumber\\
    &\left|\hat{\mu}_{j,W_{1}}-\hat{\mu}_{j,W_{2}}\right|\ge \epsilon_{\text{cut}}^{\delta}
        \Big\}
\end{align*}
where the constant $F_{1}$ is defined in~\eqref{eq:F_1}. Define the overall event
\begin{align}\label{eq:event_y}
    \mathcal{Y}=\bigcup_{t\in [T_{\text{gradual}}], j\in[K]} \mathcal{Y}_{j}(t).
\end{align}

\begin{lemma}[Upper bound on the number of resets within a gradual segment]\label{lem:reset_number}
Under the conditions that $b=T^{-d}(d>0)$, let $F_1$ denote the constant in \eqref{eq:F_1} and $\mathcal{Y}$ denote the event defined in~\eqref{eq:event_y}. Then,
\begin{align}\label{eq:possibility_y}
    \Pr(\mathcal{Y}) < \frac{2K}{T}\cdot F_{1}.
\end{align}
Under the complement event $\mathcal{Y}^{c}$, the number of resets occurring within any gradual segment is bounded by
\begin{align}\label{eq:N_m}
    N_{m} < \frac{X_{T_{m}} - Y_{T_{m}}}{F_{1}b^{-\frac{2}{3}}}.
\end{align}
where $N_{m}$ denotes the number of resets between $Y_{T_{m}}$ and $X_{T_{m}}$.
\end{lemma}

\begin{proof}
We first prove the inequality \eqref{eq:possibility_y}.
Let 
\begin{align}\label{eq:W_F1}
    \mathcal{W}_{F_{1}} = \left\{ W_{0} \in \mathcal{W} : |W_{0}| \le F_{1}b^{-\frac{2}{3}} \right\}
\end{align}
denote the set of all windows whose size are at most $F_{1}b^{-\frac{2}{3}}$.  
According to the proof of Lemma~\ref{lem:Hoeffding's inequality}, the cardinality of $\mathcal{W}_{F_{1}}$ satisfies
\begin{align*}
|\mathcal{W}_{F_{1}}|\leq \sum_{t \in T_{\text{gradual}}} t \cdot F_{1}b^{-\frac{2}{3}}< TF_{1}b^{-\frac{2}{3}}.
\end{align*}
For any fixed window $W \in \mathcal{W}_{F_{1}}$ and arm $i \in [K]$, Hoeffding’s inequality implies that
\begin{align*}
        \Pr\left(\left|\hat{\mu}_{i,W} - \mu_{i,W}\right| > \sqrt{\frac{\log(\eta^{-1})}{2|W_{i}|}}\right)\le 2\eta.
\end{align*}

Let
\begin{align*}
    \cS^c=\bigcup_{i \in [K]}\bigcup_{W^{\prime} \in \mathcal{W}_{F_{1}}}\set{\left|\hat{\mu}_{i,W} - \mu_{i,W}\right| > \sqrt{\frac{\log\left(T^{2+d}\right)}{2\left|W_{i}\right|}}}.
\end{align*}
Similar for the proof of Lemma~\ref{lem:Hoeffding's inequality} and substituting $\eta^{-1}=T^{2+d}$ into $\cS^c$, we obtain that the event $\cS^c$ occurs with probability at most
\begin{align*}
2\eta \cdot K \cdot |\mathcal{W}_{F_{1}}|\le \frac{2K}{T^{2+d}} \cdot T F_{1} b^{-\frac{2}{3}}=\frac{2K}{T^{1+d}} \cdot F_{1} b^{-\frac{2}{3}}.
\end{align*}
Since $b = T^{-d} < 1$, it follows that
\begin{align*}
\frac{2K}{T^{1+d}} \cdot F_{1} b^{-\frac{2}{3}}=\frac{2K}{T} \cdot F_{1} b^{\frac{1}{3}}< \frac{2K}{T} \cdot F_{1}
\end{align*}

Therefore, the event $\mathcal{S}$
\begin{align}\label{eq:event_s}
    \mathcal{S}=\bigcap_{i \in [K]}\bigcap_{W^{\prime} \in \mathcal{W}_{F_{1}}}\set{\left|\hat{\mu}_{i,W} - \mu_{i,W}\right| \le \sqrt{\frac{\log\left(T^{2+d}\right)}{2\left|W_{i}\right|}}}
\end{align}
holds with probability at least $1-\frac{2K}{T} F_{1}$.

We next show that, under $\mathcal{S}$, the event $\mathcal{Y}$ never occurs. We will prove this by contradiction. 
Assuming that $\left|\hat{\mu}_{j,W_{1}}-\hat{\mu}_{j,W_{2}}\right|\ge \epsilon_{\text{cut}}^{\delta}$. Applying the triangle inequality, we can obtain
\begin{align}\label{eq:triangle inequality_event_s}
\epsilon_{\text{cut}}^{\delta}\le&\left|\hat{\mu}_{j,W_{1}}-\hat{\mu}_{j,W_{2}}\right|\nonumber\\
\le& \left|\hat{\mu}_{j,W_{1}} - \mu_{j,W_{1}}\right| +\left|\mu_{j,W_{1}} - \mu_{j,W_{2}}\right|\nonumber\\
    &+\left|\hat{\mu}_{j,W_{2}} - \mu_{j,W_{2}}\right|
\end{align}
Meanwhile, event $\mathcal{S}$ implies that
\begin{align}
    &\left|\hat{\mu}_{j,W_{1}} - \mu_{j,W_{1}}\right| \le \sqrt{\frac{\log\left(T^{2+d}\right)}{2\left|W_{j,1}\right|}}, \label{eq:inequality1_event_s} \\
    &\left|\hat{\mu}_{j,W_{2}} - \mu_{j,W_{2}}\right| \le \sqrt{\frac{\log\left(T^{2+d}\right)}{2\left|W_{j,2}\right|}},\label{eq:inequality2_event_s}
\end{align}
holds for any arm $j \in [K]$, any time $t \in T_{\text{gradual}}$ and any split $W_{1}\bigcup W_{2}=W(t)$ with $W_{1},W_{2}\in\mathcal{W}_{F_{1}}$ where $\mathcal{W}_{F_{1}}$ is defined in~\eqref{eq:W_F1}.

Let $A=F_{1}b^{-\frac{2}{3}}$. By the definition of gradual change, the difference between the expected rewards over two adjacent windows $W_{1}$ and $W_{2}$ satisfies
\begin{align}\label{eq:inequality3_event_s}
    \left|\mu_{j,W_{1}} - \mu_{j,W_{2}}\right|\le 2bA.
\end{align}
Substituting ~\eqref{eq:inequality1_event_s},~\eqref{eq:inequality2_event_s} and~\eqref{eq:inequality3_event_s} into~\eqref{eq:triangle inequality_event_s}, we obtain:
\begin{align*}
\epsilon_{\text{cut}}^{\delta}\le\sqrt{\frac{\log\left(T^{2+d}\right)}{2\left|W_{j,1}\right|}}+2bA+\sqrt{\frac{\log\left(T^{2+d}\right)}{2\left|W_{j,2}\right|}}.
\end{align*}
According to our algorithm design, the threshold $\epsilon_{\text{cut}}^{\delta}$ is defined as
\begin{align*}
    \epsilon_{\text{cut}}^{\delta}=\sqrt{\frac{\log\left(T^{3}\right)}{2\left|W_{j,1}\right|}}+\sqrt{\frac{\log\left(T^{3}\right)}{2\left|W_{j,2}\right|}}.
\end{align*}
Thus, we obtain:
\begin{align}\label{eq:inequality4_event_s}
\epsilon_{\text{cut}}^{\delta}&=\sqrt{\frac{\log\left(T^{3}\right)}{2\left|W_{j,1}\right|}}+\sqrt{\frac{\log\left(T^{3}\right)}{2\left|W_{j,2}\right|}} \nonumber \\
&\le\sqrt{\frac{\log\left(T^{2+d}\right)}{2\left|W_{j,1}\right|}}+2bA+\sqrt{\frac{\log\left(T^{2+d}\right)}{2\left|W_{j,2}\right|}}.
\end{align}
Rearranging terms for~\eqref{eq:inequality4_event_s}, we have
\begin{align*}
\frac{\sqrt{3}-\sqrt{2+d}}{\sqrt{2}} \cdot (\sqrt{\frac{\log T}{\left|W_{j,1}\right|}}+\sqrt{\frac{\log T}{\left|W_{j,2}\right|}})\le 2bA.
\end{align*}
Since $|W_{j,1}|, |W_{j,2}| \le A$, it follows that
\begin{align*}
&\frac{\sqrt{3}-\sqrt{2+d}}{\sqrt{2}} \cdot2\sqrt{\frac{\log T}{A}} \nonumber \\
&\le \frac{\sqrt{3}-\sqrt{2+d}}{\sqrt{2}} \cdot (\sqrt{\frac{\log T}{\left|W_{j,1}\right|}}+\sqrt{\frac{\log T}{\left|W_{j,2}\right|}})\le 2bA.
\end{align*}
Therefore, the event $\mathcal{Y}$ holds if the following inequality satisfies
\begin{align*}
\frac{\sqrt{3}-\sqrt{2+d}}{\sqrt{2}} \cdot\sqrt{\log T}\le bA^{\frac{3}{2}}=F_{1}^{\frac{3}{2}},
\end{align*}
which implies
\begin{align*}
F_{1}\ge (\frac{\sqrt{3}-\sqrt{2+d}}{\sqrt{2}}\sqrt{\log T})^{\frac{2}{3}}.
\end{align*}
It contradicts the fact that $F_{1}=(\frac{\sqrt{3}-\sqrt{2+d}}{2\sqrt{2}}\sqrt{\log T})^{\frac{2}{3}}$. Hence, under the event $\mathcal{S}$, the event $\mathcal{Y}$ cannot occur.
Consequently, we have
\begin{align*}
    \Pr(\mathcal{Y}) < \frac{2K}{T}\cdot F_{1}
\end{align*}
Moreover, under the complement event $\mathcal{Y}^{c}$, the time interval between two resets must be greater than $F_{1}b^{-\frac{2}{3}}$. Thus, the total number of resets within each gradual segment is bounded by
\begin{align*}
    &N_{m}<(X_{T_{m}}-Y_{T_{m}})/F_{1}b^{-\frac{2}{3}}.
\end{align*}

\end{proof}

\section{Proof of Theorem~\ref{th:gradual}}
\subsection{Key probabilistic events and initial decomposition}

As discussed above, the expected regret incurred during gradual segments $\mathbb{E}[\text{Reg}(T_{\text{gradual}})]$ can be expressed as
\begin{align*}
\E[\text{Reg}(T_{\text{gradual}})] = \E \Big[\sum_{m=1}^{M+1} \sum_{t = Y_{T_{m}}}^{X_{T_{m}}} \text{Reg}(t) +\sum_{t = X_{T_{M+1}}}^{T} \text{Reg}(t)\Big]. 
\end{align*}
According to Lemma~\ref{lem:reset_number}, under the event $\mathcal{Y}^{c}$, 
the number of resets in each gradual segment is bounded as
\begin{align*}
N_{m}<(X_{T_{m}}-Y_{T_{m}})/F_{1}b^{-\frac{2}{3}}.
\end{align*}
We denote all reset times within the interval $[Y_{T_{m}}, X_{T_{m}}]$ by $L_{m,1},L_{m,2},\dots L_{m,N_{m}}$. Without loss of generality, we assume $L_{m,0}=Y_{T_{m}}$ and $L_{m,N_{m}+1}=X_{T_{m}}$. 

We let the following quantity
\begin{align*}
\epsilon_{m,n}^{(3)}(t)=\max_{L_{m,n} \le s\le t\le L_{m,n+1}}\max_{i\in[K]}|\mu_{i,s}-\mu_{i,L_{m,n}}|,
\end{align*}
denote the maximum drift in the expected rewards within the subinterval $[L_{m,n},L_{m,n+1}]$.
By Lemma~\ref{lem:diffenrence}, there exists a constant $c_g>0$ such that the event
\begin{align*}
    \mathcal{Z} = \bigcap_{t \in [T_{\text{gradual}}]} \Bigg\{ \epsilon_{m,n}^{(3)}(t) \le \frac{\log T}{c_{g}} &\left( 2bKN + 8\sqrt{\frac{\log (T^{3})}{2N}} \right) \\
&+ b\log T \Bigg\}
\end{align*}
holds with probability at least $1-\frac{2K}{T}$.
Consequently, the joint event $\mathcal{Z} \cap \mathcal{Y}^{c}$ holds with probability at least
\begin{align*}
\Pr(\mathcal{Z} \cap \mathcal{Y}^{c})&=1-\Pr(\mathcal{Z}^{c})-\Pr(\mathcal{Y})+\Pr(\mathcal{Z}^{c} \cap \mathcal{Y}) \\
&>1-\Pr(\mathcal{Z}^{c})-\Pr(\mathcal{Y})>1-\frac{2K}{T}(1+F_{1}).
\end{align*}

Accordingly, the expected regret within gradual reset intervals can be decomposed as
\begin{align}\label{eq:G1 and G2}
\mathbb{E}[\text{Reg}(T_{\text{gradual}})]&=\mathbb{E}[\setIn{\mathcal{Z}^{c} \cup \mathcal{Y}}\text{Reg}(T_{\text{gradual}})] \nonumber \\
&+\mathbb{E}[\setIn{\mathcal{Z} \cap \mathcal{Y}^{c}}\text{Reg}(T_{\text{gradual}})]\nonumber\\
&\triangleq G_1 + G_2.
\end{align}

\subsection{Bounding $G_{1}$}

We first derive an upper bound on $G_1$. Following the same steps as in the analysis of $R_{1}$ in the proof of Theorem~\ref{th:abrupt}, and using the definition of $F_{1}$ in~\eqref{eq:F_1}, we obtain
\begin{align}\label{eq:G1_upper bound}
G_1<\sqrt{\frac{2+p}{2}} \cdot T \cdot \frac{2K}{T}(1+F_{1})=\mathcal{O}\Big((\log T)^{\frac{1}{3}}\Big).
\end{align}

\subsection{Decomposing $G_2$ into $I_{1}$ and $I_{2}$}

We now turn to the upper bound of $G_{2}$, which can be decomposed as
\begin{align}\label{eq:I1 and I2}
G_2&=\mathbb{E}\Big[\setIn{\mathcal{Z} \cap \mathcal{Y}^{c}}\sum_{m=1}^{M+1} \sum_{t = Y_{T_{m}}}^{X_{T_{m}}} \text{Reg}(t)\Big] \nonumber \\
&+\mathbb{E}\Big[\setIn{\mathcal{Z} \cap \mathcal{Y}^{c}}\sum_{t = X_{T_{M+1}}}^{T} \text{Reg}(t)\Big]\nonumber\\
&\triangleq I_1 + I_2.
\end{align}

\subsection{Bounding $I_{1}$}

\subsubsection{Decomposition of $I_1$ into $I_3$ over subintervals}
By the linearity of expectation, we can further decompose $I_1$ as
\begin{align}\label{eq:I1}
I_1&=\mathbb{E}\Big[\setIn{\mathcal{Z} \cap \mathcal{Y}^{c}}\sum_{m=1}^{M+1} \sum_{t = Y_{T_{m}}}^{X_{T_{m}}} \text{Reg}(t)\Big] \nonumber \\
    &=\sum_{m=1}^{M+1}\sum_{n=0}^{N_{m}}\mathbb{E}\Big[\setIn{\mathcal{Z} \cap \mathcal{Y}^{c}} \sum_{t = L_{m,n}}^{L_{m,n+1}} \text{Reg}(t)\Big]\nonumber\\
    &=\sum_{m=1}^{M+1}\sum_{n=0}^{N_{m}} I_3.
\end{align}

Although $I_3$ depends on $m$ and $n$, we abuse notation and write $I_3$ in a way that suppresses this dependence when the meaning is clear from context.

\subsubsection{Decomposing $I_{3}$ into $I_{4}$ and $I_{5}$}
Let $\mu_{i,L_{m,n}}$ denote the expected reward of arm $i$ at time $L_{m,n}$, and define the corresponding gap as 
\begin{align*}
\Delta_{i,m,n}^{(3)}=\max_{j} \mu_{j,L_{m,n}} - \mu_{i,L_{m,n}}.
\end{align*}
Similar for the proof of Lemma~\ref{lem:regret_gap_relation1}, the instantaneous regret $\text{Reg}(t)$ and the gap $\Delta_{i,m,n}^{(3)}$ satisfy the following relationship:
\begin{align*}
|\text{Reg}(t)-\Delta_{i,m,n}^{(3)}| \le  2\epsilon_{m,n}^{(3)}(t).
\end{align*}
Let $N_{i,m,n}^{(3)} = \sum_{t=L_{m,n}}^{L_{m,n+1}} \setIn{I(t)=i}$ denote the number of times arm $i$ is pulled between $L_{m,n}$ and $L_{m,n+1}$, and let $\mathcal{H}_{m,n}^{(3)}$ 
denote the natural filtration (history information) until time $L_{m,n}$. Following a similar method as in~\eqref{eq:two_bound3} from Theorem~\ref{th:abrupt}, we can decompose the expected regret $I_3$ as
\begin{align}\label{eq:I4 and I5}
I_3&= \mathbb{E}\Big[\mathbb{E}\Big[\setIn{\mathcal{Z} \cap \mathcal{Y}^{c}}
        \sum_{L_{m,n}}^{L_{m,n+1}} \text{Reg}(t) \mid \mathcal{H}_{m,n}^{(3)}\Big]\Big] \nonumber \\
    &\le \mathbb{E}\Big[\sum_{i:\Delta_{i,m,n}^{(3)}>0}
        \min\Big\{\Delta_{i,m,n}^{(3)}\mathbb{E}\big[N_{i,m,n}^{(3)}\mid\mathcal{H}_{m,n}^{(3)}\big],
            \nonumber \\
    &\mathcal{O}\Big(\frac{\log T}{\Delta_{i,m,n}^{(3)}}\Big)\Big\}\Big]\nonumber \\
    &+ 2\mathbb{E}\Big[\sum_{i:\Delta_{i,m,n}^{(3)}>0}\mathbb{E}\Big[N_{i,m,n}^{(3)}\max_{t}\epsilon_{m,n}^{(3)}(t)\mid \mathcal{H}_{m,n}^{(3)}\Big]\Big]\nonumber\\
    &\triangleq I_4 + I_5.
\end{align}

Similar as $I_{3}$, for simplicity of notation, we write $I_{4}$ and $I_{5}$ without explicitly indicating their dependence on $m$, suppressing this dependence when the context makes it clear.
\subsubsection{Bounding $\sum_{m=1}^{M+1}\sum_{n =0}^{N_{m}}I_4$}
For the upper bound of $\sum_{m=1}^{M+1}\sum_{n =0}^{N_{m}}I_4$, similar for analysis of the proof in Theorem~\ref{th:abrupt}, we have 
\begin{align*}
    &\mathbb{E}\Big[\sum_{i:\Delta_{i,m,n}^{(3)}>0} \min\Big\{\Delta_{i,m,n}^{(3)}\mathbb{E}\left[N_{i,m,n}^{(3)}\mid\mathcal{H}_{m,n}^{(3)}\right],\mathcal{O}\Big(\frac{\log T}{\Delta_{i,m,n}^{(3)}}\Big)\Big\}\Big] \nonumber \\
    &<\sum_{i} \mathcal{O}\Big(\sqrt{\mathbb{E}\Big[N_{i,m,n}^{(3)}\Big]\log T}\Big).
\end{align*}
Since $\sum_{i} N_{i,m,n}^{(3)} = L_{m,n+1} - L_{m,n}$, applying the Cauchy-Schwarz inequality yields
\begin{align*}
&\sum_{i} \mathcal{O}\Big(\sqrt{\mathbb{E}\Big[N_{i,m,n}^{(3)}\Big]\log T}\Big)\le \nonumber \\
&\mathcal{O}\Big(\sqrt{K(L_{m,n+1}-L_{m,n})\log T}\Big).
\end{align*}
Therefore,
\begin{align}\label{eq:I4_upper bound}
I_4<\mathcal{O}\Big(\sqrt{K(L_{m,n+1}-L_{m,n})\log T}\Big).
\end{align}

Next, summing over all subintervals $n$ within the $m$-th gradual segment and invoking Lemma~\ref{lem:reset_number}, we obtain
\begin{align*}
     &\sum_{n=0}^{N_{m}} I_4<\sum_{n=0}^{N_{m}} \mathcal{O}\Big(\sqrt{K(L_{m,n+1}-L_{m,n})\log T}\Big) \nonumber \\
     &<\sum_{n =0}^{\frac{(X_{T_{m}}-Y_{T_{m}})}{F_{1}b^{-\frac{2}{3}}}} \mathcal{O}\Big(\sqrt{K(L_{m,n+1}-L_{m,n})\log T}\Big) \nonumber \\
     &<\mathcal{O}\Big(\sqrt{K(X_{T_{m}}-Y_{T_{m}})[\frac{(X_{T_{m}}-Y_{T_{m}})}{F_{1}b^{-\frac{2}{3}}}+1]\log T}\Big).
\end{align*}
Finally, by summing over all gradual segments and noting that $\sum_{m=1}^{M+1}(X_{T_{m}}-Y_{T_{m}})<T$, we further apply the Cauchy-Schwarz inequality to obtain
\begin{align*}
    &\sum_{m=1}^{M+1}\mathcal{O}\Big(\sqrt{K(X_{T_{m}}-Y_{T_{m}})[\frac{(X_{T_{m}}-Y_{T_{m}})}{F_{1}b^{-\frac{2}{3}}}+1]\log T}\Big) \nonumber \\
    &<\mathcal{O}\Big(\sqrt{K(M+1)T(T/F_{1}b^{-\frac{2}{3}}+1)\log T}\Big).
\end{align*}
Recalling that $F_1 = \mathcal{O}\Big((\log T)^{1/3}\Big)$, and that $K$ and $M$ are constants, while $b = T^{-d}$, we have
\begin{align*}
    &\mathcal{O}\Big(\sqrt{K(M+1)T(T/F_{1}b^{-\frac{2}{3}}+1)\log T}\Big) \nonumber \\
    &<\mathcal{O}\Big(\sqrt{(\log T)^{\frac{2}{3}}T^{2-\frac{2}{3}d}+T\log T}\Big).
\end{align*}
Thus, we get
\begin{align}\label{eq:sum_I4_upper bound}
\sum_{m=1}^{M+1}\sum_{n =0}^{N_{m}}I_4<\mathcal{O}\Big(\sqrt{(\log T)^{\frac{2}{3}}T^{2-\frac{2}{3}d}+T\log T}\Big).
\end{align}

\subsubsection{Bounding $\sum_{m=1}^{M+1}\sum_{n =0}^{N_{m}}I_5$}
Next, we derive the upper bound of $\sum_{m=1}^{M+1}\sum_{n=0}^{N_{m}}I_5$. Recall that~\eqref{eq:I4 and I5} is conditioned on event $\mathcal{Z} \cap \mathcal{Y}^{c}$. According to the event $\mathcal{Z}$, the drift within each subinterval $[L_{m,n},L_{m,n+1}]$ satisfies
\begin{align}\label{eq:epsilon_{m,n}^{(3)} upper bound}
\epsilon_{m,n}^{(3)}(t) \le \frac{\log T}{c_{g}} \Big(2bKN + 8\sqrt{\frac{\log (T^{3})}{2N}}\Big)+b \log T.
\end{align}
Substituting the assumption $N=\mathcal{O}((bK)^{-\frac{2}{3}})$ into~\eqref{eq:epsilon_{m,n}^{(3)} upper bound} yields
\begin{align*}
&\frac{\log T}{c_{g}} \Big(2bKN + 8\sqrt{\frac{\log (T^{3})}{2N}}\Big)+b \log T \nonumber \\
&=\mathcal{O}\Big((bK)^{\frac{1}{3}}(\log T)^{\frac{3}{2}}\Big).
\end{align*}
Under the assumption $b=T^{-d}$ and given that $K$ is constant, we further obtain
\begin{align}\label{eq:epsilon_{m,n}^{(3)} upper bound1}
    &\epsilon_{m,n}^{(3)}(t) \le\mathcal{O}\Big((bK)^{\frac{1}{3}}(\log T)^{\frac{3}{2}}\Big) =\mathcal{O}\Big(T^{-\frac{d}{3}}(\log T)^{\frac{3}{2}}\Big).\nonumber \\
\end{align}
Following a similar method as in the proof of Theorem~\ref{th:abrupt}, we have
\begin{align*}
&I_5=2\mathbb{E}\Big[\sum_{i:\Delta_{i,m,n}^{(3)}>0} \mathbb{E}\Big[N_{i,m,n}^{(3)}\max_{t}\epsilon_{m,n}^{(3)}(t)\mid\mathcal{H}_{m,n}^{(3)}\Big]\Big] \nonumber \\
&< 2\mathcal{O}\Big(T^{-\frac{d}{3}}(\log T)^{\frac{3}{2}}\Big) \mathbb{E}\Big[\sum_{i} \mathbb{E}\Big[N_{i,m,n}^{(3)}\mid\mathcal{H}_{m,n}^{(3)}\Big]\Big].
\end{align*}
Applying the law of total expectation and using $\sum_{i} N_{i,m,n}^{(3)} = L_{m,n+1}-L_{m,n}$, we obtain
\begin{align}\label{eq:I5_upper bound}
    &I_5<2\mathcal{O}\Big(T^{-\frac{d}{3}}(\log T)^{\frac{3}{2}}\Big) \mathbb{E}\Big[\sum_{i} \mathbb{E}\Big[N_{i,m,n}^{(3)}\mid\mathcal{H}_{m,n}^{(3)}\Big]\Big] \nonumber \\
    &=2\mathcal{O}\Big(T^{-\frac{d}{3}}(\log T)^{\frac{3}{2}}\Big) \mathbb{E}\Big[\sum_{i} N_{i,m,n}^{(3)}\Big] \nonumber \\
    &=2\mathcal{O}\Big(T^{-\frac{d}{3}}(\log T)^{\frac{3}{2}}\Big)\Big(L_{m,n+1}-L_{m,n}\Big).
\end{align}
Summing $I_5$ over all subintervals $n$ and noting that $\sum_{n=0}^{N_{m}}(L_{m,n+1}-L_{m,n}) = X_{T_{m}}-Y_{T_{m}}$, we have
\begin{align*}
    &\sum_{n=0}^{N_{m}} I_5 <\sum_{n=0}^{N_{m}}2\mathcal{O}\Big(T^{-\frac{d}{3}}(\log T)^{\frac{3}{2}}\Big)\Big(L_{m,n+1}-L_{m,n}\Big) \nonumber \\
    &=2\mathcal{O}\Big(T^{-\frac{d}{3}}(\log T)^{\frac{3}{2}}\Big)\Big(X_{T_{m}}-Y_{T_{m}}\Big).
\end{align*}
Finally, since $\sum_{m =1}^{M+1}(X_{T_{m}}-Y_{T_{m}})<T$, we obtain
\begin{align}\label{eq:sum_I5_upper bound}
    &\sum_{m =1}^{M+1}\sum_{n=0}^{N_{m}} I_5<2\mathcal{O}\Big(T^{-\frac{d}{3}}(\log T)^{\frac{3}{2}}\Big)\cdot T \nonumber \\
    &=\mathcal{O}\Big(T^{1-\frac{d}{3}}(\log T)^{\frac{3}{2}}\Big).
\end{align}

\subsubsection{Combining Bounds for $I_1$}
According to~\eqref{eq:I1} and~\eqref{eq:I4 and I5}, combining~\eqref{eq:sum_I4_upper bound} with~\eqref{eq:sum_I5_upper bound}, we obtain
\begin{align}\label{eq:I1_upper bound}
    &I_1\le\sum_{m =1}^{M+1}\sum_{n=0}^{N_{m}}(I_4+I_5)<\mathcal{O}\Big(\sqrt{(\log T)^{\frac{2}{3}}T^{2-\frac{2}{3}d}+T\log T}\Big) \nonumber \\
    &+\mathcal{O}\Big(T^{1-\frac{d}{3}}(\log T)^{\frac{3}{2}}\Big).
\end{align}

\subsection{Bounding $I_2$}

\subsubsection{Decomposing $I_{2}$ into $I_{6}$ and $I_{7}$}
Next, we solve the upper bound of $I_2$. Similarly, we let $\mu_{i,X_{T_{M+1}}}$ denote the expected reward of arm $i$ at time $X_{T_{M+1}}$, and define the corresponding gap as 
\begin{align}\label{eq:Delta_i,m_2}
    \Delta_{i,M+1}^{(4)}=\max_{j} \mu_{j,X_{T_{M+1}}} - \mu_{i,X_{T_{M+1}}}.
\end{align}
Let the following quantity
\begin{align}\label{eq:epsilon_{M+1}^{4}(t)}
\epsilon_{M+1}^{(4)}(t) = \max_{X_{T_{M+1}}\le s \le t \le T} \max_{i\in[K]} | \mu_{i,s} - \mu_{i,X_{T_{M+1}}}|,
\end{align}
denote the maximum amount of drift within the interval $[X_{T_{M+1}},T]$. Similar for the proof of Lemma~\ref{lem:regret_gap_relation1}, the instantaneous regret $\text{Reg}(t)$ and the gap $\Delta_{i,M+1}^{(4)}$ satisfy the following relationship:
\begin{align}\label{eq:Diff_Reg and Delta2}
|\text{Reg}(t)-\Delta_{i,M+1}^{(4)}| \le  2\epsilon_{M+1}^{(4)}(t).
\end{align}
Let $N_{i,M+1}^{(4)} = \sum_{t=X_{T_{M+1}}}^{T} \setIn{I(t)=i}$ denote the number of times arm $i$ is pulled between $X_{T_{M+1}}$ and $T$, and let $\mathcal{H}_{M+1}^{(4)}$ denote the natural filtration (history information) until time $X_{T_{M+1}}$. Similar for~\eqref{eq:two_bound3}, we decompose $I_2$ into the following two terms:
\begin{align}\label{eq:I6 and I7}
&I_2\le \nonumber \\
     &\mathbb{E}\Big[\sum_{i:\Delta_{i,M+1}^{(4)}>0} \min\Big\{\Delta_{i,M+1}^{(4)}\mathbb{E}\left[N_{i,M+1}^{(4)}\mid\mathcal{H}_{M+1}^{(4)}\right],\nonumber \\
     &\mathcal{O}\Big(\frac{\log T}{\Delta_{i,M+1}^{(4)}}\Big)\Big\}\Big] \nonumber \\
     &+ 2\mathbb{E}\Big[\sum_{i:\Delta_{i,M+1}^{(4)}} \mathbb{E}\Big[N_{i,M+1}^{(4)}\max_{t}\epsilon_{M+1}^{(4)}(t)\mid\mathcal{H}_{M+1}^{(4)}\Big]\Big]\nonumber\\
     &\triangleq I_6 + I_7.
\end{align}

\subsubsection{Bounding $I_{6}$}
For the upper bound of $I_6$, we repeat the same steps used in the derivation of~\eqref{eq:I4_upper bound}. Noting that $K$ is a constant, we obtain
\begin{align}\label{eq:I6_upper bound}
    &I_6<\sum_{i}\mathcal{O}\Big(\sqrt{\mathbb{E}\left[N_{i,M+1}^{(4)}\right]\log T}\Big) \nonumber \\
    &<\mathcal{O}\Big(\sqrt{K\Big(T-X_{T_{M+1}}\Big)\log T}\Big)<\mathcal{O}\Big(\sqrt{T\log T}\Big).
\end{align}

\subsubsection{Bounding $I_{7}$}
For the term $I_7$, we use the same argument as in the derivation of the upper bound of $I_5$. Under the conditions $N=\mathcal{O}\big((bK)^{-2/3}\big)$ and $b=T^{-d}$, one shows (as in~\eqref{eq:epsilon_{m,n}^{(3)} upper bound1}) that
\begin{align*}
    \epsilon_{M+1}^{(4)}(t)\le \mathcal{O}\Big(T^{-\frac{d}{3}}(\log T)^{\frac{3}{2}}\Big).
\end{align*}
Hence, by linearity of expectation and $\sum_{i} \mathbb{E}[N_{i,M+1}^{(4)}]=T-X_{T_{M+1}}$,
\begin{align}\label{eq:I7_upper bound}
    &I_7<2\mathcal{O}\Big(T^{-\frac{d}{3}}(\log T)^{\frac{3}{2}}\Big) \mathbb{E}\Big[\sum_{i} N_{i,M+1}^{(4)}\Big] \nonumber \\
    &=2\mathcal{O}\Big(T^{-\frac{d}{3}}(\log T)^{\frac{3}{2}}\Big)\Big(T-X_{T_{M+1}}\Big)  \nonumber \\&<\mathcal{O}\Big(T^{1-\frac{d}{3}}(\log T)^{\frac{3}{2}}\Big). 
\end{align}

\subsubsection{Combining Bounds for $I_2$}
Combining~\eqref{eq:I6_upper bound} and~\eqref{eq:I7_upper bound} yields
\begin{align}\label{eq:I2_upper bound}
    &I_2\le I_6+I_7<\mathcal{O}\Big(\sqrt{T\log T}\Big)+\mathcal{O}\Big(T^{1-\frac{d}{3}}(\log T)^{\frac{3}{2}}\Big).
\end{align}

\subsection{Combining all terms and concluding the bound}

Finally, recalling~\eqref{eq:I1 and I2} and combining ~\eqref{eq:I1_upper bound} with ~\eqref{eq:I2_upper bound}, we obtain for $G_2$:
\begin{align}\label{eq:G2_upper bound}
    &G_2=I_1+I_2<\mathcal{O}\left(\sqrt{(\log T)^{\frac{2}{3}}T^{2-\frac{2}{3}d}+T\log T}\right) \nonumber \\
    &+\mathcal{O}\Big(T^{1-\frac{d}{3}}(\log T)^{\frac{3}{2}}\Big).
\end{align}
Finally, combining~\eqref{eq:G1_upper bound} with~\eqref{eq:G2_upper bound}, the expected regret incurred during gradual reset intervals is bounded as
\begin{align}
    &\mathbb{E}[\text{Reg}(T_{\text{gradual}})]<\mathcal{O}\left(\sqrt{(\log T)^{\frac{2}{3}}T^{2-\frac{2}{3}d}+T\log T}\right) \nonumber \\
    &+\mathcal{O}\Big(T^{1-\frac{d}{3}}(\log T)^{\frac{3}{2}}\Big).\nonumber \\
\end{align}

\end{document}